\documentclass{article}

\PassOptionsToPackage{numbers}{natbib}

\usepackage[preprint]{neurips_2026}

\usepackage{amsmath}
\usepackage[utf8]{inputenc} 
\usepackage[T1]{fontenc}    
\usepackage{hyperref}       
\usepackage{url}            
\usepackage{booktabs}       
\usepackage{amsfonts}       
\usepackage{nicefrac}       
\usepackage{microtype}      
\usepackage{xcolor}         
\usepackage{subcaption}

\usepackage[english]{babel}
\let\cite\citep

\hypersetup{
  colorlinks,
  allcolors={blue!50!black}
}

\usepackage{amssymb}
\usepackage{mathtools}
\usepackage{amsthm}
\usepackage{float}
\usepackage{algorithm}
\usepackage{algorithmic}
\usepackage{multirow}

\theoremstyle{plain}
\newtheorem{theorem}{Theorem}[section]
\newtheorem*{theorem*}{Theorem}

\theoremstyle{definition}
\newtheorem{definition}[theorem]{Definition}

\theoremstyle{remark}

\newcommand{\prob}{\mathbb{P}}
\newcommand{\basematrix}{X}
\newcommand{\score}{\hat{\mu}}
\newcommand{\truescore}{\mu}

\title{Rank Intervals for Leaderboards: \\A Hierarchical Framework for Model Evaluation}

\author{%
  Bitya Neuhof \\
  Department of Statistics and Data Science\\
  The Hebrew University of Jerusalem \\
  \texttt{bitya.neuhof@mail.huji.ac.il} \\
  \And
  Yuval Benjamini \\
  Department of Statistics and Data Science \\
  The Hebrew University of Jerusalem \\
  \texttt{yuval.benjamini@mail.huji.ac.il} \\
}

\begin{document}

\maketitle

\begin{abstract}
    Pretrained models are often evaluated on multi-task leaderboards to measure their applicability in diverse contexts. However, current methods for aggregating performance across tasks into leaderboard-level rankings do not address the uncertainty and variability at the task level. While recent works have proposed interval-based model rankings, the principled aggregation of uncertainty from individual tasks to leaderboard-level rankings remains unaddressed, and variation in models' performance across tasks is frequently obscured. In this work, we introduce a hierarchical framework that constructs model rank intervals with statistical guarantees at both levels: task-level rank confidence intervals from pairwise comparisons, and leaderboard-level rank prediction intervals using a conformal approach. This enables reliable quantification of model rank for each observed task and for new potential tasks. Experiments on simulated data and the TabArena and PromptEval (MMLU) benchmarks show that our method yields statistically valid and informative intervals, enabling reliable, uncertainty-aware model ranking on leaderboards.
\end{abstract}

\section{Introduction}

The practice of model ranking has been formalized through the establishment of public leaderboards, such as the Hugging Face Open LLM Leaderboard~\cite{open-llm-leaderboard-v2}, which provides an overview of models' rankings on diverse tasks, and Chatbot Arena~\cite{chiang2024chatbot}, which summarizes pairwise preferences between models. Leaderboards have become the primary mechanism for researchers to report empirical progress, with rankings indicating advancement and encouraging competition. When choosing a model, it is important to consider that performance can vary significantly across tasks, especially for large pretrained models (such as LLMs) built for broad use instead of excelling at one task.

Leaderboards typically aggregate model performance across tasks using methods like weighted averages or by modeling win rates with Bradley-Terry (BT)~\cite{bradley1952rank}. These summary metrics, often presented as rankings, help compare and select models. However, they may obscure substantial variation in performance across tasks, thereby missing important model–task interactions. Additionally, both the estimation of single-task performance and the aggregation process introduce uncertainty to the rankings, but this uncertainty is rarely made explicit. As leaderboards expand to include more tasks, models, and metrics, making meaningful comparisons becomes even more difficult.

While recent work has introduced statistical methods to analyze pairwise comparisons~\cite{ameli2024statistical, valdeira2025ranking} or aggregate scores~\cite{longjohn2025statistical, ackerman2025statistical}, these approaches often overlook the direct quantification of uncertainty in the leaderboard rankings~\cite{foucart2025ranking}. Specifically, they lack a unified mechanism to propagate the variability observed at the individual task level~\cite{carmona2025towards, geburek2024lmems} into the ranking in the leaderboard level. This absence of guarantees makes it difficult to determine whether differences in model rankings are consistent across tasks or merely artifacts of the task variability~\cite{mishra2021robust, heineman2025signal}. To address these limitations and systematically quantify ranking uncertainty, we introduce a hierarchical ranking framework aligned with standard benchmarking pipelines. At the task level, models are compared pairwise to construct rank confidence intervals (CIs) for each model on each task, representing the plausible range of true rankings for that model on the given task. We then aggregate these task-level rank CIs to obtain leaderboard-level rank prediction intervals (PIs), which reflect the expected rank range for a model on a new, unseen task drawn from the same distribution. Thus, the two levels answer distinct inferential questions: task-level CIs quantify uncertainty for specific tasks, while leaderboard-level PIs summarize overall ranking variability and help predict model performance on future tasks.

Our contributions are as follows: (1) We introduce a hierarchical uncertainty framework for leaderboard rankings that accounts for both task variability and model correlations. (2) We propose a novel conformal method to compute leaderboard rank PIs by aggregating task-level CIs. (3) We empirically validate our framework using simulated leaderboards and demonstrate its effectiveness in comparing machine learning (ML) models on tabular data (TabArena~\cite{erickson2025tabarena}) and LLMs on language understanding tasks (MMLU~\cite{hendryckstest2021}). (4) We highlight ranking variability and recommend visualizations and measures that improve transparency and interpretability, enabling users to distinguish whether a model's poor ranking results from consistently lower performance or from high variability across tasks. To facilitate reproducibility and future research, we release our \href{https://github.com/BityaNeuhof/leaderboard-rank-intervals}{ranking framework and evaluation code}.

\section{Task-level model ranking}\label{sec:task_ranking}

\subsection{Terminology and definitions}

In this section, we describe a ranking method for comparing the performance of all models on a single task. Suppose we have $M$ candidate models, $c_1, \dots, c_M$, each representing a fitted statistical or ML model designed to perform a specific task (such as regression, classification, or generation). 
For a task $t_b$, we define a corresponding \emph{base values matrix} $\mathbf{\basematrix}^b \in \mathbb{R}^{n_b \times M}$, where $n_b$ denotes the size of the task $t_b$, since tasks may vary in size. Each entry $\basematrix^b_{ij}$ in this matrix corresponds to the evaluation of model $c_j$ for a specific minimal data unit. Depending on the context, a base value could be a performance value for an observation in a dataset, or the score on a cross-validation fold. We assume that higher scores indicate better performance.
From $\mathbf{\basematrix}^b$, we calculate the mean performance vector $\mathbf{\score}^b = \{ \score^b_1, \ldots, \score^b_M \}$ for each model $c_j$ on task $t_b$, and the estimated covariance matrix $\hat{\Sigma^b}$, which captures both the variances of models and the covariances between models.

Because evaluation metrics may use different numerical scales, model comparisons are typically based on rankings derived from these scores rather than the raw values themselves~\cite{demvsar2006statistical, longjohn2025statistical}. The observed ranks of the models for task $t_b$, denoted as  $\mathbf{\hat{r}}^b = (\hat{r}^b_1, \ldots, \hat{r}^b_M)$ with $\hat{r}^b_j \in \{1, \ldots, M\}$, are determined by ordering the mean performance scores $\mathbf{\score}^b$. The model with the highest mean score receives rank $M$, while the lowest receives rank $1$. 

The observed ranks are uncertain, and this uncertainty arises from the estimation noise in the observed performance scores~\cite{rising2021uncertainty}. Let $\truescore^b_j$ represent the \emph{true performance score} of model $c_j$ on task $t_b$. The observed performance score $\score^b_j$ is considered an unbiased but noisy estimate of $\truescore^b_j$. Our objective is to understand how this estimation noise influences the resulting model rankings. Unlike the observed (noisy) ranks, we define the \emph{true ranks} $\mathbf{r^b} = (r^b_1, \ldots, r^b_M)$ based on the vector of true scores $(\truescore^b_1, \ldots, \truescore^b_M)$. To model potential ties in the true ranks, we follow the definition of rank set from~\citet{al2022simultaneous}.

\begin{definition}(Rank Set)
\label{def:rankset}
    Define the lower rank of $\truescore^b_j$ as $
        l^b_j = 1 + \#\{k: \truescore^b_j > \truescore^b_k, \, j \neq k\}$
    and the upper rank as $
        u^b_j = M - \#\{k: \truescore^b_j < \truescore^b_k, \, j \neq k\} $. 
    Then the rank set of $\truescore^b_j$ is $r^b_j = \{l^b_j, \ldots, u^b_j\}$.
\end{definition}

When there are no ties in the true performance scores, the lower and upper ranks are equal, so the rank set contains only one value, aligned with the standard notion of a unique rank. Throughout this paper, we refer to the true ranks (denoted as $\mathbf{r^b}$) as the collection of rank sets defined above. Although noise makes exact ties in observed performance scores rare, assuming ties in true ranks is useful when models are equivalent, meaning their true performance scores are indistinguishable given the noise.

\subsection{Task-Level rank confidence intervals}

The target of inference is the set of true ranks for all models with respect to task $t_b$. 
Let $([L^b_1, U^b_1], \ldots, [L^b_M, U^b_M])$ denote the rank intervals of the true ranks of $M$ models on task $t_b$.
\begin{definition}(Marginal Coverage)
An interval of ranks for model $c_j$ is said to have \emph{marginal coverage} rate of $1 - \alpha_{tsk}$ on task $t_b$ if the interval is a valid CI for its corresponding true rank:
\begin{equation}\label{eq:marginal_cover}
        \begin{aligned}
        \prob \big( r^b_j \subseteq [L^b_j, U^b_j]\big) \geq 1 - \alpha_{tsk}.
        \end{aligned}
    \end{equation}
\end{definition}
We aim to construct task-level rank CIs with marginal coverage for all models.
Note that marginal coverage is not sufficient to support selection after ranking~\cite{benjamini2005false}.

\subsection{Constructing task-level rank confidence intervals}

Most methods for constructing rank CIs rely on pairwise location tests~\cite{holm2013confidence, al2022simultaneous}, and adjusted coverage rate to control ranking error at the model level or across all models.
In a pairwise test, for each pair of models $c_j$ and $c_k$ on a task $t_b$, we test two one-sided hypotheses:
\begin{equation}\label{eq:hypotheses}
    \begin{aligned}
        &H_{jk;0}^b: \truescore^b_j \geq \truescore^b_k \text{ vs } H_{jk;1}^b: \truescore^b_j < \truescore^b_k, \text{ and }
        &H_{kj;0}^b: \truescore^b_k \geq \truescore^b_j \text{ vs } H_{kj;1}^b: \truescore^b_k < \truescore^b_j.
    \end{aligned}
\end{equation}
The result of each hypothesis test is a p-value $p_{jk}^b$ for every ordered pair of models $(c_j, c_k)$ on task $t_b$. For $M$ models, there are $M(M-1)$ such pairwise comparisons. The hypotheses are directional.

Rank CIs for the true ranks are obtained by \emph{counting} rejections of hypotheses after controlling for family wise error  (FWER) for decisions associated with a single model at a time \cite{holm2013confidence}. Algorithm \ref{alg:task_ci} summarizes the construction of marginal
 task-level rank CIs from the pairwise comparison.

\begin{algorithm}
\caption{Task-Level Rank CI}\label{alg:task_ci}
\begin{algorithmic}[1]
\STATE {\bfseries Input:}
\STATE Pairwise p-values $\{p_{jk}^b : j \neq k\}$;
\STATE FWER-controlling procedure $\phi$ (e.g., Holm), and miscoverage rate $\alpha_{tsk}$;
\FOR{$j = 1,\dots,M$}
    \STATE {\bfseries Lower rank bound:}
    \STATE Apply $\phi$ to $\{p_{kj}^b : k \neq j\}$ at level $\alpha_{tsk}/2$,
           obtaining rejection vector $R_{j,L} \in \{0,1\}^{M-1}$.
    \STATE $L_j^b \gets 1 + \sum_{k=1}^{M-1} R_{j,L}(k)$
    \COMMENT{Count models significantly worse than $c_j$}

    \STATE {\bfseries Upper rank bound:}
    \STATE Apply $\phi$ to $\{p_{jk}^b : k \neq j\}$ at level $\alpha_{tsk}/2$,
           obtaining rejection vector $R_{j,U} \in \{0,1\}^{M-1}$.
    \STATE $U_j^b \gets M - \sum_{k=1}^{M-1} R_{j,U}(k)$
    \COMMENT{Count models significantly better than $c_j$}
\ENDFOR
\STATE Return $\big\{[L_1^b, U_1^b], \ldots , [L_M^b, U_M^b]\big\}$
\end{algorithmic}
\end{algorithm}

\begin{theorem}\label{thm:task}
[Holm  2013] The set of rank CIs $[L^b_1, U^b_1], \ldots, [L^b_M, U^b_M]$ constructed by Algorithm~\ref{alg:task_ci}, has marginal coverage rate of $(1-\alpha_{tsk})$.
\end{theorem}
For the formal statement and proof, see Theorem 1 in~\citet{holm2013confidence}. 
In the analyses presented in this paper, we use a paired t-test for significance testing and Holm's procedure for multiplicity control.

Other methods for constructing valid rank CIs, which provide either marginal or stronger guarantees of simultaneous coverage, can also be applied at the task level to generate rank intervals for each task. Examples include the methods of~\citet{ al2021simultaneous, al2022simultaneous, chetverikov2024csranks, chandra2025finite, valdeira2025ranking}. We discuss the implementation choices in Appendix~\ref{app:implementation_details}. Our experiments (Appendix~\ref{app:synthetic_add}) show that bootstrap uncertainty estimates do not produce valid CIs in the presence of ties.
Our ranking method compares models' performance scores, evaluating each model on the same data. To construct rank CIs, we assume a score and a measure of variability for each model or pair of models. Thus, our framework applies to pairwise preference data, such as Chatbot Arena~\cite{chiang2024chatbot}. See Appendix~\ref{app:pairwise_preferences} for details and an example.
\section{Model ranking across tasks}\label{sec:leaderboard_ranking}
In this section, we describe how to aggregate task-level rank CIs to infer model rankings at the \emph{leaderboard level}. 
We introduce a novel leaderboard-level interval for each model, based on the quantiles of the task-level rank CIs. The leaderboard-level interval is a PI, in the sense that it guarantees a specified coverage-probability for the true ranks of new tasks.

\subsection{Terminology and definitions}
Let $t_1, \dots, t_N$ be a set of $N$ distinct tasks, with all models $c_1, \dots, c_M$ evaluated on each task. Let $(1 - \alpha_{ldb})$ denote the desired coverage rate. The objective is to obtain an interval that would cover the true rank of a new task with high probability.

\begin{definition}(Marginal Task Coverage)
Assume $t_1, \ldots, t_N \sim \mathcal{P}$, where $\mathcal{P}$ is a distribution over tasks, and let $[L_j, U_j]$ denote the leaderboard-level rank PI for model $c_j$, estimated from the $N$ observed tasks. We say that $[L_j, U_j]$ maintains \emph{marginal task coverage} for model $c_j$ if:
\begin{equation}\label{eq:marginal_task_coverage}
    \begin{aligned}
        &\prob_{t_1,...,t_N,t^*\sim \mathcal{P}} \big(r_j(t^*) \subseteq [L_j, U_j]\big)
        &\geq 1 - \alpha.
    \end{aligned}
\end{equation}
Here, $t^*$ represents a new, unseen task, and $r_j(t^*)$ is the true rank of model $c_j$ on this task.
\end{definition}
\textbf{Interpretation:} Fix a model $c_j$, and draw a set of $N+1$ tasks. If you obtain the interval from tasks $t_1,...,t_N$, then it will cover the true rank of $t^*$ with probability of at least $(1 - \alpha).$

The main formal assumption for obtaining intervals with marginal task coverage is that new tasks are drawn from the same distribution as the observed tasks used to estimate the leaderboard-level rank PIs. This assumption is common in conformal inference. In fact, the leaderboard-level rank PIs we propose can be viewed as conformal PIs constructed using the task-level rank CIs~\cite{liu2025conformalintervals}.

\subsection{Leaderboard-level rank prediction intervals}

Let $\mathcal{I}_{M \times N}$ denote the collection of $N$ sets of ($1-\alpha_{tsk}$) task-level rank CIs for the true ranks of $M$ models, with one set per task:
\begin{equation}\label{eq:ci_collection}
    \begin{aligned}
        \mathcal{I}_{M \times N} \coloneq \{([L^1_1, U^1_1], \ldots, [L^1_M, U^1_M]), \ldots,
        ([L^N_1, U^N_1], \ldots, [L^N_M, U^N_M])\}.
    \end{aligned}
\end{equation}
To aggregate task-level rank CIs into leaderboard-level rank PIs while retaining finite-sample guarantees, we use tools from conformal inference, which provide distribution-free, model-agnostic methods for calibrating uncertainty sets~\cite{vovk2005algorithmic}.

For model $c_j$, let $\mathbf{L}_j = sorted(L^1_j, \ldots, L^N_j )$ and $\mathbf{U}_j=sorted(U^1_j, \ldots, U^N_j)$ denote the sorted vectors of the lower and upper bounds from the $N$ task-level rank CIs, so that $\mathbf{L}_j[k]$ denote the $k$-th largest lower bound. We also set $\mathbf{L}_j[0] = 1$ and $\mathbf{U}_j[N+1] = M$. Following~\citet{lei2018distribution}, we define the leaderboard-level rank PI for model $c_j$ using the inflated $\frac{\alpha_{ldb}}{2}$ quantile of the lower bounds and the inflated $1 - \frac{\alpha_{ldb}}{2}$ quantile of the upper bounds. 
\begin{definition}(Interval Quantiles)\label{def:quantile}
Denote:
\begin{equation}\label{eq:quantile_ind}
        \begin{aligned}
        k_l = \left\lfloor(N+1)\frac{\alpha_{ldb}}{2} \right\rfloor, \quad
        k_u = \left\lceil (N+1)(1-\frac{\alpha_{ldb}}{2}) \right\rceil.
    \end{aligned}
\end{equation}
Then:
\begin{equation}\label{eq:quantile_val}
    \begin{aligned}
        L_j = \mathbf{L}_j\big[k_l\big], \quad
        U_j = \mathbf{U}_j\big[k_u\big].
    \end{aligned}
\end{equation}
Here, $\mathbf{L}_j[k_l]$ and $\mathbf{U}_j[k_u]$ denote the $k_l$-th and $k_u$-th elements from the sorted vectors of lower and upper bounds.
\end{definition}

Algorithm~\ref{alg:quantile} aggregates task-level rank CIs into leaderboard-level rank PIs using quantiles. For each model $c_j$, the algorithm selects quantiles of the lower and upper interval bounds to ensure the interval $[L_j, U_j]$ achieves the target coverage $(1- \alpha)$ across tasks, for $\alpha = \alpha_{tsk} + \alpha_{ldb}$. To avoid degenerate leaderboard-level rank PIs, $\alpha_{ldb}$ should not be too small; setting $\alpha_{ldb} \geq \frac{2}{N+1}$ ensures $k_l \geq 1$ and $k_u \leq N$. To prevent trivial intervals $[1, M]$, we require at least $N \geq 3$ tasks.

\begin{algorithm}[ht]
  \caption{Quantile Merge}
  \label{alg:quantile}
  \begin{algorithmic}[1]
    \STATE {\bfseries Input:} \\
    \qquad Collection of task-level rank CIs $\mathcal{I}_{M \times N}, \quad N \geq 3$; \\
    \qquad Leaderboard-level miscoverage rate $\alpha_{ldb} \in \left(\frac{2}{N+1}, 1\right)$.
    \STATE $k_l \gets \left\lfloor(N+1)\frac{\alpha_{ldb}}{2} \right\rfloor$; $k_u \gets \left\lceil (N+1)(1-\frac{\alpha_{ldb}}{2}) \right\rceil$
    \FOR{$j=1$ {\bfseries to} $M$}
        \STATE $\mathbf{L}_j \gets sorted(L^1_j, \ldots, L^N_j )$; $\mathbf{U}_j \gets sorted(U^1_j, \ldots, U^N_j)$
        \STATE $L_j=\mathbf{L}_j[k_l]$; $U_j=\mathbf{U}_j[k_u]$
    \ENDFOR
    \STATE Return $\big\{[L_1, U_1], \ldots, [L_M, U_M]\big\}$
  \end{algorithmic}
\end{algorithm}

\begin{theorem}\label{thm:leaderboard}
Assume the observed tasks $t^1, \ldots, t^N$ are independently sampled from the distribution $\mathcal{P}$. Construct the collection $\mathcal{I}_{M \times N}$ of task-level rank CIs, each with marginal coverage $1 - \alpha_{tsk}$.
Now, consider a new independent task $t^* \sim \mathcal{P}$, with true ranks $r_1^*, \ldots, r_M^*$.
The leaderboard-level rank PIs $([L_1, U_1], \ldots, [L_M, U_M])$ produced by Algorithm~\ref{alg:quantile} achieve marginal task coverage of $1 - \alpha$, for $\alpha = \alpha_{ldb} + \alpha_{tsk}$. That is, for each model $c_j$,
    \begin{equation}\label{eq:thm_leaderboard}
        \begin{aligned}
            \prob \big( r_j^* \subseteq [L_j, U_j]\big)
            \geq 1 - \alpha.
        \end{aligned}
    \end{equation}
\end{theorem}

The coverage guarantee is $1 - (\alpha_{ldb} + \alpha_{tsk})$ reflecting the possibility of task-level ranking error.

\textbf{Proof Overview:} Focusing on model $c_j$, we construct a task-level rank CI $[L_j^*, U_j^*]$ for the new task $t^*$, which guarantees coverage of the true rank $r_j^{*}$ with probability at least $1-\alpha_{tsk}$. The lower bound $L_j^*$ is exchangeable with the $N$ observed lower bounds for model $c_j$, ensuring that all $N+1$ lower bounds are equally likely. Similarly, for the upper bounds, all $N+1$ values are equally likely, so $P\big(L^*_j \leq L_j \big) \leq 1 - \frac{\alpha_{ldb}}{2}$, and $P\big(U^*_j \geq U_j \big) \leq 1 - \frac{\alpha_{ldb}}{2}$. By accounting for the potential coverage errors associated with $L_j^*$ and $U_j^*$, we derive the overall coverage bound. A detailed proof is provided in Appendix~\ref{app:proof}.

Leaderboard-level rank PIs provide a summary of a model's performance across all preselected tasks, without any selection based on specific tasks. As a result, it is sufficient for task-level CIs to achieve marginal coverage to guarantee marginal coverage at the leaderboard level. Note that Algorithm~\ref{alg:quantile} outputs rank PIs, which are different than rank CIs. While both quantify ranking uncertainty, they differ in the inferential question and interpretation; see Appendix~\ref{app:interp} for details.

\section{Experiments}\label{sec:experiments}

\subsection{Synthetic data simulations}
We examine how uncertainty propagates from task-level rank CIs to leaderboard-level rank PIs, and assess Algorithm~\ref{alg:quantile}'s ability to control ranking errors through marginal coverage. We validate the framework using synthetic data that mimics a leaderboard, evaluating multiple models across tasks. This lets us control model correlations and generate identically distributed tasks. 

\paragraph{Synthetic data generation process}
We represent the true model scores as a vector, one entry per model, with $M \in \{10, 30\}$. Model correlations are encoded in a covariance matrix $\Sigma$ with equal variances and a block correlation structure. For each task (where the number of tasks $N \in \{20, 60\}$), we sample a per-task true score vector ($\mathbf{\truescore}^b$) from a normal distribution centered at the models' true scores.
Base values for each task-model pair are generated by adding noise to the true scores. Complete details of the data-generating process and simulation parameters are provided in Appendix~\ref{app:synthetic_desc}. True ranks for each task are assigned based on the true scores. We vary the number of base values ($\{10, 200\}$) to simulate both cross-validation folds and data-observation scenarios. Each parameter configuration is repeated 100 times in the simulation.

\paragraph{Ranking baselines}
At the task level, we compare the CIs constructed by Algorithm~\ref{alg:task_ci} with CIs produced by a bootstrap rank-aggregation baseline (Algorithm~\ref{alg:task_bootstrap}). At the leaderboard-level, we compare Algorithm~\ref{alg:quantile} to a baseline method that merges rank CIs by union, where $L_j = \min_{b=1}^N(L_j^b)$ and $U_j = \max_{b=1}^N(U^b_j)$. This union approach is commonly used as a baseline in conformal inference~\cite{gasparin2024merging}.

\paragraph{Evaluation}
For each parameter configuration, we evaluate both components of our framework: task-level rank CIs and leaderboard-level PIs. We assess rank intervals using two key measures: (1) \textbf{Average normalized width}, which is the average width of rank intervals normalized to fall between 0 and 1; and (2) \textbf{Coverage}, the proportion of experimental runs in which the true ranks fall within their corresponding CIs or PIs. Precise definitions for both metrics at the task and leaderboard levels are provided in Appendix~\ref{app:synthetic_desc}. Low width values correspond to narrower intervals, indicating better performance. The average observed coverage should be at least $1- (\alpha_{tsk} -\alpha_{ldb})$. Coverage values significantly above $1- \alpha_{ldb}$ indicate conservativeness.

\paragraph{Results}
We evaluate the width and coverage of rank CIs and PIs across parameter configurations, varying task-level ($\alpha_{tsk}$) and leaderboard-level ($\alpha_{ldb}$) miscoverage rates. Comparing our FWER-based task-level rank CIs to bootstrap CIs shows that bootstrap CIs fail to maintain coverage in the presence of ties (coverage with ties $0.661 \pm 0.099$, and without ties $0.984 \pm 0.018$). Therefore, we use only our task-level rank CIs (coverage with ties $0.984 \pm 0.009$, and without ties $0.999 \pm 0.001$) to obtain leaderboard-level rank PIs. See Appendix~\ref{app:synthetic_add} for detailed results. This analysis highlights the importance of using valid methods for rank CIs. Figure~\ref{fig:summary_by_alpha_N}(a) summarizes the comparison between Algorithm~\ref{alg:quantile} (quantile) and the baseline (union) method. Coverage is reported as the average for unseen tasks (Equation~\ref{eq:app_ldb_cover_unseen}) across repetitions. We also examined how the number of tasks affects PIs' width. Figure~\ref{fig:summary_by_alpha_N}(b) shows results for multiple $N$ values, assuming moderate correlation ($\rho=0.5$, $block\_size=3$) and no ties. As expected from conformal intervals, the quantile method's interval width remains constant as the number of tasks increases, while the union method's width grows. Appendix~\ref{app:synthetic_add} provides detailed tables comparing both methods under varying correlations and ties.

\begin{figure}[ht]
    \centering
    \begin{subfigure}[c]{0.49\linewidth}
        \centering 
        \includegraphics[width=0.95\linewidth]{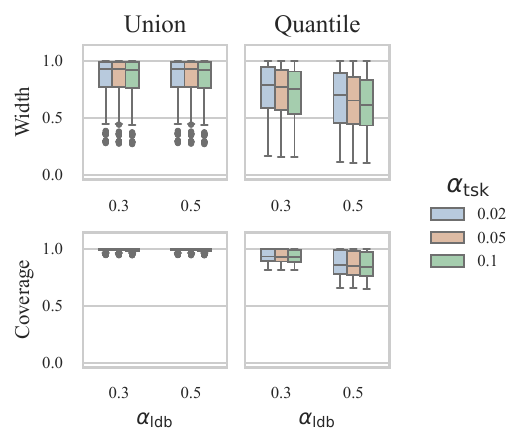}
        \caption{Summary by miscoverage rates.}
    \end{subfigure}
    \hfill 
    \begin{subfigure}[c]{0.49\linewidth}
         \centering 
         \includegraphics[width=0.98\linewidth]{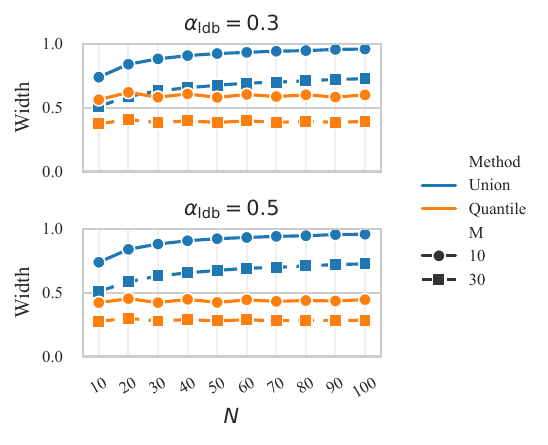}
         \caption{Impact of the number of tasks ($N$).}
     \end{subfigure}
    \caption{Hierarchical framework performance. (a) Distribution of width and coverage across simulated configurations, where the quantile method is less conservative than the union baseline. (b) Normalized width for different $N$ values; for the quantile method, width remains stable with $N$.} 
    \label{fig:summary_by_alpha_N}
\end{figure}

\subsection{TabArena}\label{sec:tabarena}

The TabArena leaderboard~\cite{erickson2025tabarena} is a benchmark providing an overview of the performance of multiple ML prediction models across various tabular datasets, including regression, classification, and multi-class classification. More details are provided in Appendix~\ref{app:tabarena_info}. We use the cross-validation scores of 44 models across 51 datasets to demonstrate how to use and interpret the output of our hierarchical ranking framework. In terms of our framework, each dataset is a task.

\paragraph{Leaderboard overview}
In the TabArena leaderboard, models are ranked by Elo scores. Although CIs are provided for the Elo scores, there is no explicit quantification of uncertainty in the rankings. In Figure~\ref{fig:tabarena_elo_pi} we present the Elo CIs (based on TabArena-v0.1 Leaderboard in table A.1 in~\citet{erickson2025tabarena}), side by side with our leaderboard-level PIs obtained by Algorithm~\ref{alg:quantile} with $\alpha_{tsk}=0.05$ and $\alpha_{ldb}=0.5$. The models are ordered by the Elo score. Visualizing the Elo CIs of all models already improves interpretability and awareness of potential ties between models, as it is clear that intervals overlap. The Elo CIs bound the average performance across all datasets. In contrast, the rank PIs quantify the expected performance of models on new tasks. To demonstrate the differences, we colored the intervals of the KNN and CAT models. While both views agree that KNN(D) is the worst, the Elo view ranks the KNN(T+E) higher than 5 other models. As shown in Figure~\ref{fig:tabarena_elo_pi}, the rank PIs indicate that all versions of KNN consistently ranked worst across all datasets. Therefore, they are expected to be the worst on new tasks as well.

\begin{figure}
    \centering
    \includegraphics[width=0.98\linewidth]{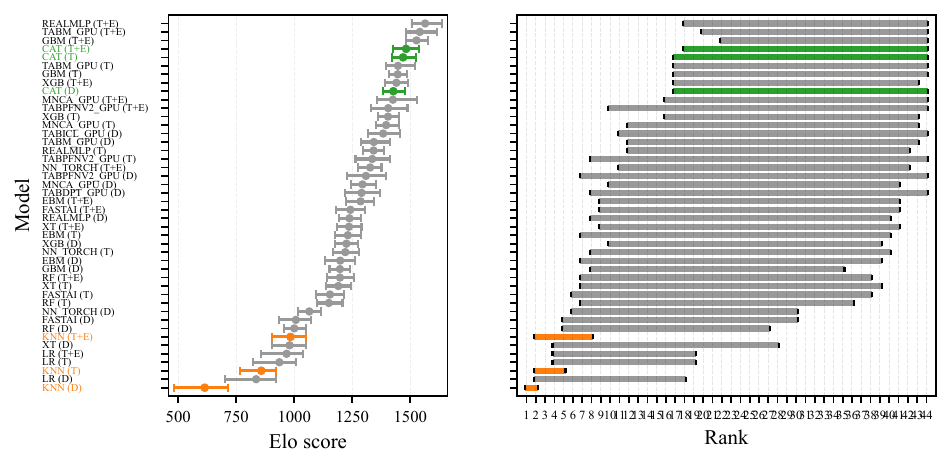}
    \caption{Elo scores with 95\% CIs (left), and leaderboard-level PIs (right) for all models. Narrow intervals indicate stable rankings across tasks. Wide intervals may result from consistently wide task-level intervals or from task heterogeneity.}
    \label{fig:tabarena_elo_pi}
\end{figure}

\begin{figure}[ht]
     \centering
     \includegraphics[width=0.98\linewidth]{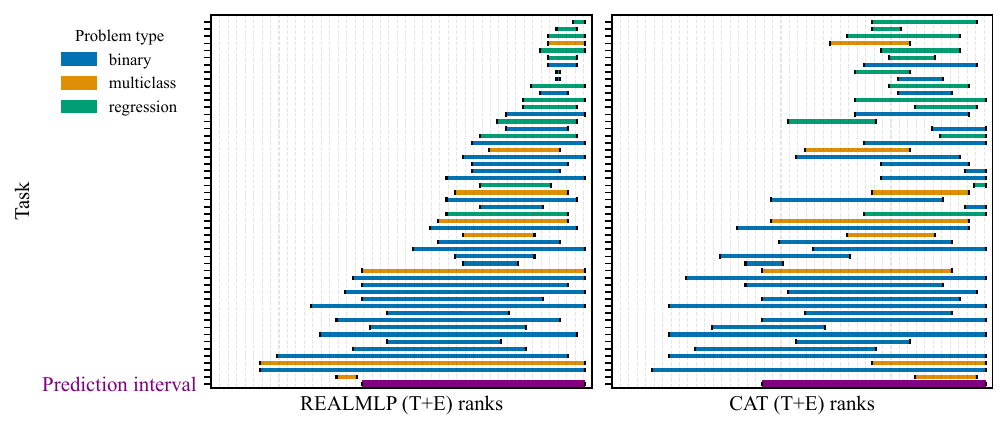}
     \caption{Dataset-level rank CIs for two models that share identical leaderboard-level PIs. While CAT (right) exhibits wide intervals across most datasets, REALMLP (left) shows highly variable performance, with narrow CIs in some datasets and wide CIs in others. This demonstrates how similar leaderboard-level PIs can mask differences in model behavior across datasets.}
     \label{fig:tabarena_task_dist}
\end{figure}

\paragraph{Dataset rank CIs distribution per-model}
Beyond leaderboard aggregation, examining dataset-level rank CIs reveals detailed insights about individual model performance. Figure~\ref{fig:tabarena_task_dist} presents dataset-level rank CIs for two top-ranked models from Figure~\ref{fig:tabarena_elo_pi}(b), colored by problem type. While both models have identical leaderboard-level PIs, their distributions of dataset-level rank CIs differ. For example, REALMLP has narrower rank CIs for regression tasks, indicating more reliable performance in this setting, while CAT shows narrower intervals for multiclass problems. Thus, if faced with a new regression dataset, REALMLP is likely to outperform CAT. Additionally, examining specific datasets can be informative: for the customer satisfaction in airline dataset~\cite{stojanovic2026evaluating}, the rank CI width is 1 for REALMLP versus 6 for CAT, demonstrating greater rank stability for REALMLP.

\subsection{Massive multitask language understanding}\label{sec:mmlu}

The Massive Multitask Language Understanding (MMLU) benchmark~\cite{hendryckstest2021} evaluates the knowledge and problem-solving capabilities of LLMs in different subjects. We utilize the PromptEval version of this benchmark \cite{polo2024efficient}, which provides a correctness matrix for 100 prompt variations across 15 models and 57 subjects, with at least 100 questions per subject. We treat each subject as a task. Further details regarding the benchmark subjects and models is in Appendix~\ref{app:mmlu_info}. For each subject, we average over prompt variants to obtain accuracy score per question. We then construct subject-level rank CIs ($\alpha_{tsk}=0.05$, examples in Figure ~\ref{fig:mmlu_ranks_examples_questions} b) and leaderboard-level rank PIs ($\alpha_{ldb}=0.5$, Figure  ~\ref{fig:app_mmlu_questions_quantile}
 a).

\paragraph{Subject-level analysis}
In a complex benchmark such as MMLU, it is insufficient to rank models based on a single average accuracy across subjects, as this fails to account for subject-to-subject variability.
Figure~\ref{fig:mmlu_ranks_examples_questions}(a) shows the distribution of accuracy for questions of two subjects, abstract algebra and high school macroeconomics, and the subject-level rank CIs (b). The models are ordered by average accuracy for both subjects, yet this average poorly represents these two specific subjects, let alone the full suite of 57. Furthermore, note the substantial variability in the individual questions of the algebra subject. The rank CIs indicate that the observed ranking cannot be trusted and all models are statistically interchangeable. While many intervals overlap for both subjects, in macroeconomics we  identify several differentiated groups of models.

\begin{figure}
    \centering
    \begin{subfigure}[r]{0.98\linewidth}
        \centering
        \includegraphics[width=0.98\linewidth]{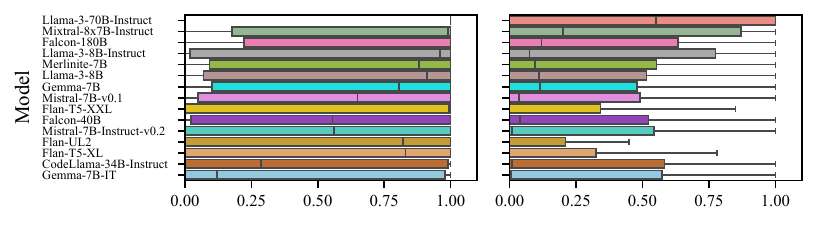}
        \caption{Accuracy distribution.}
    \end{subfigure}
    \begin{subfigure}[r]{0.98\linewidth}
        \centering
        \includegraphics[width=0.98\linewidth]{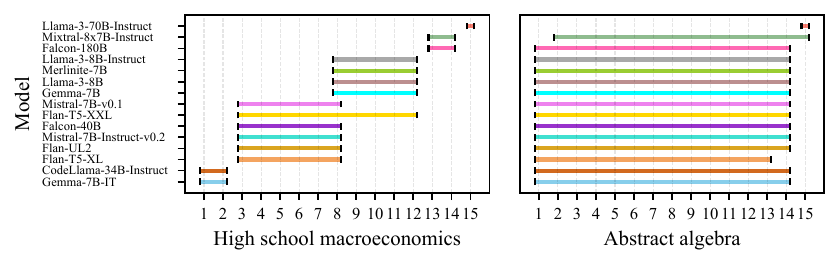}
        \caption{Rank CIs.}
    \end{subfigure}
    \caption{Examples of subject-level accuracy distribution and rank CIs. The intervals for high school macroeconomics (left) are narrower and show a clear groups of models, whereas for abstract algebra models almost completely overlap, indicating no detectable difference between models.}
    \label{fig:mmlu_ranks_examples_questions}
\end{figure}

\paragraph{Leaderboard-level analysis}

To illustrate how PIs summarize performance across subjects, we randomly select 10 subjects as unseen test cases and aggregate the remaining 47 subject-level rank CIs into leaderboard-level rank PIs. Results are shown in Figure~\ref{fig:app_mmlu_questions_quantile}(a). For each new subject, we check whether its rank CI falls within the corresponding rank PI. The average coverage over these unseen subjects is at least $1 - (\alpha_{tsk} + \alpha_{ldb}) = 0.45$ for all models. Figure~\ref{fig:app_mmlu_questions_quantile}(b) presents two example models' PIs alongside the rank CIs of the unseen subjects. For Llama-3-70B-Instruct, the  rank CIs are wider than the rank PI due to the miscoverage rate $\alpha_{ldb}=0.5$.

\begin{figure}
    \centering
    \begin{subfigure}[r]{0.98\linewidth}
    \centering
    \includegraphics[width=0.98\linewidth]{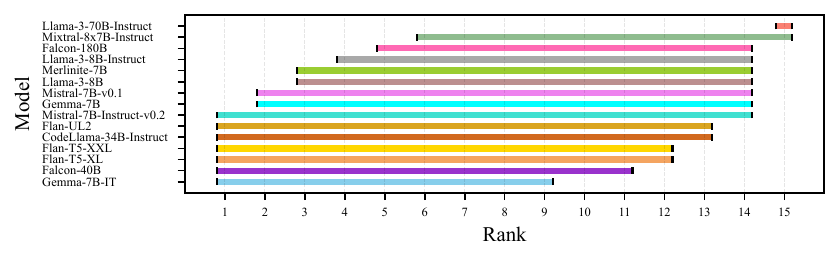}
    \caption{Leaderboard-level PIs}
    \end{subfigure}
    \begin{subfigure}[r]{0.98\linewidth}
    \centering
    \includegraphics[width=0.98\linewidth]{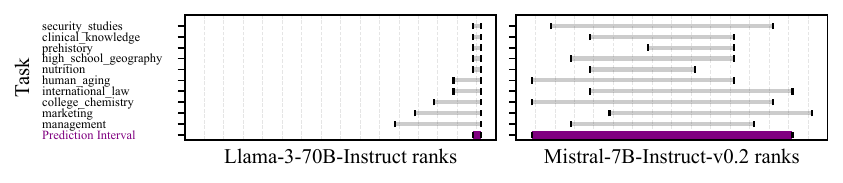}
    \caption{Rank CIs of unseen subjects}
    \end{subfigure}
    \caption{Leaderboard-level PIs for 47 subjects and rank CIs of unseen subjects. For models other than Llama-3-70B-Instruct, the ranks vary with tasks and their rank CIs often fall in the middle of the pack.}
    \label{fig:app_mmlu_questions_quantile}
\end{figure}

\section{Related work}\label{sec:related_work}

Leaderboards often rely on single-value rankings without uncertainty quantification~\cite{miller2024adding, ackerman2025statistical}. While platforms like Chatbot Arena~\cite{chiang2024chatbot} use BT models~\cite{bradley1952rank, ameli2024statistical} and bootstrapping for CIs, there is growing consensus that identifying significant model differences requires rank-based CIs rather than traditional score CIs~\cite{foucart2025ranking, valdeira2025ranking}. Robust model evaluation requires granular analysis across categories or prompts~\cite{frick2025prompt, avelar2026prompt, jung2026defines}, as simple averaging can obscure variability and lack statistical guarantees~\cite{demvsar2006statistical, longjohn2025statistical}. To address this,~\citet{demvsar2006statistical} recommended using the Friedman test with post-hoc analysis to compare multiple classifiers across multiple datasets. However, applying such tests indiscriminately can lead to misleading conclusions and reduce replicability. More robust frameworks, such as those proposed by~\citet{Dror2017} for NLP, address these issues. Advanced techniques, including linear mixed-effect models~\cite{geburek2024lmems} and Bayesian hierarchical modeling~\cite{longjohn2025statistical}, have also been proposed to quantify uncertainty in summary metrics.

Despite this progress in evaluating and quantifying model performance, existing methods generally do not address how to propagate rank uncertainty when aggregating leaderboard tasks, nor do they capture task-specific ranking uncertainty or cross-task heterogeneity. Current approaches still lack interpretable CIs that jointly account for uncertainty in aggregate leaderboard rankings while reflecting variation across tasks. To address these gaps, our work introduces a hierarchical framework that constructs task-level rank CIs and aggregates them into leaderboard-level PIs with explicit coverage guarantees over the task distribution.
\section{Conclusions}\label{sec:conclusions}

We introduced a hierarchical framework that constructs CIs for model rankings on individual tasks and aggregates them into leaderboard-level PIs, while ensuring statistical validity and coverage guarantees over the task distribution. Our method is the first to fully characterize ranking uncertainty in leaderboard evaluation by aggregating task-level CIs, operates directly on ranks, and requires minimal assumptions. A potential limitation of our approach is that it assumes all models are evaluated on all tasks; for incomplete data, methods like BT can be used, where models are evaluated on different tasks or prompts. Rank PIs are only applicable if rank CIs can be constructed for each task on the same models. For multiple metrics, we recommend aggregating them into a composite score before applying our method, as it does not natively handle multivariate metrics. Coverage at the task level is somewhat conservative, which may result in wider intervals. This can be partially mitigated by using methods that adapt to model correlations rather than Algorithm \ref{alg:task_ci}, though these are more computationally intensive. While rank-based summaries improve comparability and interpretability, they may obscure the magnitudes of scores; we suggest supplementing rank intervals with actual score summaries for a more complete evaluation.

\begin{ack}
This work was supported by the Israel Science Foundation.
\end{ack}

\bibliographystyle{plainnat}
\bibliography{leaderboard_ranking_preprint}

@misc{open-llm-leaderboard-v2,
  author = {Cl{\'e}mentine Fourrier and Nathan Habib and Alina Lozovskaya and Konrad Szafer and Thomas Wolf},
  title = {Open LLM Leaderboard v2},
  year = {2024},
  publisher = {Hugging Face},
  howpublished = {\url{https://huggingface.co/spaces/open-llm-leaderboard/open_llm_leaderboard}}
}

@book{wilcox2011introduction,
  title={Introduction to robust estimation and hypothesis testing},
  author={Wilcox, Rand R},
  year={2011},
  publisher={Academic press}
}

@incollection{wilcox2023comparing,
  title={Comparing Two Dependent Groups},
  author={Wilcox, Rand R},
  booktitle={A Guide to Robust Statistical Methods},
  pages={83--96},
  year={2023},
  publisher={Springer}
}

@inproceedings{mishra2021robust,
  title={How robust are model rankings: A leaderboard customization approach for equitable evaluation},
  author={Mishra, Swaroop and Arunkumar, Anjana},
  booktitle={Proceedings of the AAAI Conference on Artificial Intelligence},
  volume={35},
  pages={13561--13569},
  year={2021}
}

@article{miller2024adding,
  title={Adding error bars to evals: A statistical approach to language model evaluations},
  author={Miller, Evan},
  journal={arXiv preprint arXiv:2411.00640},
  year={2024}
}

@article{heineman2025signal,
  title={Signal and noise: A framework for reducing uncertainty in language model evaluation},
  author={Heineman, David and Hofmann, Valentin and Magnusson, Ian and Gu, Yuling and Smith, Noah A and Hajishirzi, Hannaneh and Lo, Kyle and Dodge, Jesse},
  journal={arXiv preprint arXiv:2508.13144},
  year={2025}
}

@inproceedings{foucart2025ranking,
  title={Ranking the scores of algorithms with confidence},
  author={Foucart, Adrien and Elskens, Arthur and Decaestecker, Christine},
  booktitle={ESANN 2025 proceedings},
  pages={431--436},
  year={2025}
}

@article{ackerman2025statistical,
  title={Statistical multi-metric evaluation and visualization of LLM system predictive performance},
  author={Ackerman, Samuel and Farchi, Eitan and Raz, Orna and Toledo, Assaf},
  journal={arXiv preprint arXiv:2501.18243},
  year={2025}
}

@article{ameli2024statistical,
  title={A statistical framework for ranking llm-based chatbots},
  author={Ameli, Siavash and Zhuang, Siyuan and Stoica, Ion and Mahoney, Michael W},
  journal={arXiv preprint arXiv:2412.18407},
  year={2024}
}

@article{wang2024confidence,
  title={Confidence Diagram of Nonparametric Ranking for Uncertainty Assessment in Large Language Models Evaluation},
  author={Wang, Zebin and Han, Yi and Fang, Ethan X and Wang, Lan and Lu, Junwei},
  journal={arXiv preprint arXiv:2412.05506},
  year={2024}
}

@article{avelar2026prompt,
  title={Prompt-Dependent Ranking of Large Language Models with Uncertainty Quantification},
  author={Avelar Menendez, Angel Rodrigo and Liu, Yufeng and Dai, Xiaowu},
  journal={arXiv e-prints},
  pages={arXiv--2603},
  year={2026}
}

@inproceedings{frick2025prompt,
  title={Prompt-to-leaderboard: Prompt-adaptive LLM evaluations},
  author={Frick, Evan and Chen, Connor and Tennyson, Joseph and Li, Tianle and Chiang, Wei-Lin and Angelopoulos, Anastasios Nikolas and Stoica, Ion},
  booktitle={Forty-second International Conference on Machine Learning},
  year={2025}
}

@article{jung2026defines,
  title={Who Defines" Best"? Towards Interactive, User-Defined Evaluation of LLM Leaderboards},
  author={Jung, Minji and Lee, Minjae and Kim, Yejin and Choi, Sarang and Kahng, Minsuk},
  journal={arXiv preprint arXiv:2604.21769},
  year={2026}
}

@article{Dror2017,
  author    = {Rotem Dror and Gili Baumer and Marina Bogomolov and Roi Reichart},
  title     = {Replicability Analysis for Natural Language Processing: Testing Significance with Multiple Datasets},
  journal   = {Transactions of the Association for Computational Linguistics},
  volume    = {5},
  pages     = {471--486},
  year      = {2017},
  url       = {https://aclanthology.org/Q17-1034},
}

@inproceedings{carmona2025towards,
  title={Towards Robust Comparisons of NLP Models: A Case Study},
  author={Carmona, Vicente Ivan Sanchez and Jiang, Shanshan and Dong, Bin},
  booktitle={Proceedings of the 31st International Conference on Computational Linguistics},
  pages={4973--4979},
  year={2025}
}

@article{longjohn2025statistical,
  title={Statistical Uncertainty Quantification for Aggregate Performance Metrics in Machine Learning Benchmarks},
  author={Longjohn, Rachel and Gopalan, Giri and Casleton, Emily},
  journal={arXiv preprint arXiv:2501.04234},
  year={2025}
}

@article{demvsar2006statistical,
  title={Statistical comparisons of classifiers over multiple data sets},
  author={Dem{\v{s}}ar, Janez},
  journal={Journal of Machine learning research},
  volume={7},
  number={Jan},
  pages={1--30},
  year={2006}
}

@article{geburek2024lmems,
  title={LMEMs for post-hoc analysis of HPO Benchmarking},
  author={Geburek, Anton and Mallik, Neeratyoy and Stoll, Danny and Bouthillier, Xavier and Hutter, Frank},
  journal={arXiv preprint arXiv:2408.02533},
  year={2024}
}

@article{erickson2025tabarena,
  title={TabArena: A Living Benchmark for Machine Learning on Tabular Data},
  author={Erickson, Nick and Purucker, Lennart and Tschalzev, Andrej and Holzm{\"u}ller, David and Desai, Prateek Mutalik and Hutter, Frank and others},
  journal={arXiv preprint arXiv:2506.16791},
  year={2025}
}

@inproceedings{stojanovic2026evaluating,
  title={Evaluating Passenger Segment Sensitivity and Reliability in Airline Satisfaction AI Systems},
  author={Stojanovic, Milo{\v{s}} and Nikolic, Milena},
  booktitle={2026 XXV International Symposium INFOTEH-JAHORINA (INFOTEH)},
  year={2026},
  month={March},
  address={Jahorina, Bosnia and Herzegovina},
  url={https://infoteh.etf.ues.rs.ba/zbornik/2026/radovi/343.pdf}
}

@article{polo2024efficient,
  title={Efficient multi-prompt evaluation of llms},
  author={Polo, Felipe M and Xu, Ronald and Weber, Lucas and Silva, M{\'\i}rian and Bhardwaj, Onkar and Choshen, Leshem and de Oliveira, Allysson F and Sun, Yuekai and Yurochkin, Mikhail},
  journal={Advances in Neural Information Processing Systems},
  volume={37},
  pages={22483--22512},
  year={2024}
}

@article{hendryckstest2021,
title={Measuring Massive Multitask Language Understanding},
author={Dan Hendrycks and Collin Burns and Steven Basart and Andy Zou and Mantas Mazeika and Dawn Song and Jacob Steinhardt},
journal={Proceedings of the International Conference on Learning Representations (ICLR)},
year={2021}
}

@inproceedings{chiang2024chatbot,
  title={Chatbot arena: An open platform for evaluating llms by human preference},
  author={Chiang, Wei-Lin and Zheng, Lianmin and Sheng, Ying and Angelopoulos, Anastasios Nikolas and Li, Tianle and Li, Dacheng and Zhu, Banghua and Zhang, Hao and Jordan, Michael and Gonzalez, Joseph E and others},
  booktitle={Forty-first International Conference on Machine Learning},
  year={2024}
}

@inproceedings{son2025arena,
    title={Arena-Lite: Efficient and Reliable Large Language Model Evaluation via Tournament-Based Direct Comparisons},
    author={Son, Seonil and Oh, Ju-Min and Jin, Heegon and Jang, Cheolhun and Jeong, Jeongbeom and Kim, Kuntae},
    booktitle={Proceedings of the 2025 Conference on Empirical Methods in Natural Language Processing},
    pages={7068--7086},
    year={2025}
}

@article{benjamini2005false,
  title={False discovery rate--adjusted multiple confidence intervals for selected parameters},
  author={Benjamini, Yoav and Yekutieli, Daniel},
  journal={Journal of the American Statistical Association},
  volume={100},
  number={469},
  pages={71--81},
  year={2005},
  publisher={Taylor \& Francis}
}

@techreport{holm2013confidence,
    author={Holm, Sture},
    title={Confidence intervals for ranks},
    institution={Uppsala University},
    year={2013}
}

@article{al2021simultaneous,
  title={Simultaneous confidence intervals for ranks using the partitioning principle},
  author={Al Mohamad, Diaa and van Zwet, Erik and Solari, Aldo and Goeman, Jelle},
  journal={Electronic Journal of Statistics},
  volume={15},
  number={1},
  pages={3109--3134},
  year={2021},
  publisher={Institute of Mathematical Statistics and Bernoulli Society},
  doi={10.1214/21-EJS1847}
}

@article{al2022simultaneous,
  title={Simultaneous confidence intervals for ranks with application to ranking institutions},
  author={Al Mohamad, Diaa and Goeman, Jelle J and van Zwet, Erik W},
  journal={Biometrics},
  volume={78},
  number={1},
  pages={238--247},
  year={2022},
  publisher={Wiley Online Library}
}

@inproceedings{valdeira2025ranking,
  title={Ranking with confidence for large scale comparison data},
  author={Valdeira, Filipa and Soares, Cl{\'a}udia},
  booktitle={Proceedings of the 2025 SIAM International Conference on Data Mining (SDM)},
  pages={223--232},
  year={2025},
  organization={SIAM}
}

@article{chetverikov2024csranks,
  title={csranks: an R package for estimation and inference involving ranks},
  author={Chetverikov, Denis and Mogstad, Magne and Morgen, Pawel and Romano, Joseph and Shaikh, Azeem and Wilhelm, Daniel},
  journal={arXiv preprint arXiv:2401.15205},
  year={2024}
}

@article{rising2021uncertainty,
  title={Uncertainty in ranking},
  author={Rising, Justin},
  journal={arXiv preprint arXiv:2107.03459},
  year={2021}
}

@article{chandra2025finite,
  title={Finite-Sample Valid Rank Confidence Sets for a Broad Class of Statistical and Machine Learning Models},
  author={Chandra, Onrina and Xie, Min-ge},
  journal={arXiv preprint arXiv:2512.00316},
  year={2025}
}

@article{liu2025conformalintervals,
      title={Prediction Sets and Conformal Inference with Interval Outcomes}, 
      author={Weiguang Liu and Áureo de Paula and Elie Tamer},
      year={2025},
      eprint={2501.10117},
      archivePrefix={arXiv},
      primaryClass={econ.EM},
      journal={arXiv preprint arXiv:2501.10117},
      url={https://arxiv.org/abs/2501.10117}, 
}

@book{vovk2005algorithmic,
  title={Algorithmic learning in a random world},
  author={Vovk, Vladimir and Gammerman, Alexander and Shafer, Glenn},
  year={2005},
  publisher={Springer}
}

@article{lei2018distribution,
  title={Distribution-free predictive inference for regression},
  author={Lei, Jing and G’Sell, Max and Rinaldo, Alessandro and Tibshirani, Ryan J and Wasserman, Larry},
  journal={Journal of the American Statistical Association},
  volume={113},
  number={523},
  pages={1094--1111},
  year={2018},
  publisher={Taylor \& Francis}
}

@techreport{anderson2017split,
  title={Split-sample strategies for avoiding false discoveries},
  author={Anderson, Michael L and Magruder, Jeremy},
  year={2017},
  institution={National Bureau of Economic Research}
}

@article{gasparin2024merging,
  title={Merging uncertainty sets via majority vote},
  author={Gasparin, Matteo and Ramdas, Aaditya},
  journal={arXiv preprint arXiv:2401.09379},
  year={2024}
}

@article{posten1979robustness,
  title={The robustness of the one-sample t-test over the Pearson system},
  author={Posten, Harry O},
  journal={Journal of Statistical Computation and Simulation},
  volume={9},
  number={2},
  pages={133--149},
  year={1979},
  publisher={Taylor \& Francis}
}

@article{bradley1952rank,
    title={Rank analysis of incomplete block designs: I. the method of paired comparisons},
    author={Bradley, Ralph Allan and Terry, Milton E},
    journal={Biometrika},
    volume={39},
    number={3/4},
    pages={324--345},
    year={1952},
    publisher={JSTOR}
}

\newpage
\appendix
\section{Ranking implementation Details}\label{app:implementation_details}

\subsection{Task-level rank CIs}
Here we provide more details regarding our method for constructing task-level rank CIs, as well as a discussion of alternative choices.

\paragraph{Statistical tests}
Throughout this paper, we use the parametric paired t-tests to compute the p-values.
Paired tests are computed on the vector of differences $d_{jk} = X^b_{\cdot j}-X^b_{\cdot k}$,
where $X^b_{\cdot j}$ is the $j$'th column in the observed performance matrix. The p-value $p^b_{jk}$ of the test
is defined as: $p^b_{jk}= T_{n_b\,-\,1} \big(\frac{\bar{d}}{s_d /\sqrt{n_b} }\big)$, where $\bar{d}$ and $s_d$ are the sample mean and standard deviation of $d_{jk}$, and $T_{n_b\,-\,1}$ is the cumulative distribution function (CDF) of the T-distribution with $n_b-1$ degrees of freedom. Formally, t-tests require that $\bar{d}$ be normally distributed.
In model performance comparison, we expect this condition to hold, because typically either $n_b$ represents data observations and is large, or each cross-validation fold is based on averaging the performance on many observations and is approximately normal. Furthermore, the t-test was shown to work well in practice for symmetric distributions~\citep{posten1979robustness} even when the normality assumption is violated.

\paragraph{FWER control procedures}
We use Holm's procedure for FWER control over a family of hypotheses.
For $K$ hypotheses with ordered $p$-values $p_{(1)}\le\cdots\le p_{(K)}$, Holm’s procedure examines p-values iteratively, starting at $k=1$. At each step, it rejects hypothesis $H_{(k)}$ if
\[
p_{(k)} \le \frac{\alpha}{K-k+1}.
\]
It stops at the first index that failed to reject. 
The Holm procedure controls the FWER at level $\alpha$,
is fast to compute, and holds for any dependency structure between the p-values. 

\paragraph{Alternative methods for constructing rank CIs}

Alternative statistical methods for constructing rank CIs require either the construction of the p-values or more efficient FWER procedures.

For computing the p-value when the samples are matched, the t-distribution can be replaced with more robust versions when the distributions are not normal, for example a Wilcoxon signed-rank test~\cite{wilcox2011introduction} (less power, more robust) or a trimmed-mean based t-test~\cite{wilcox2023comparing} (intermediate power but requires sampling). 
Non-matched samples will require different p-value computations; see our example in Appendix \ref{app:pairwise_preferences}.

For the FWER procedure, the Holm method can be replaced by resampling algorithms that compute the thresholds directly, by resampling from the data distribution after removing all mean difference~\cite{chetverikov2024csranks}. This gives thresholds that are better adapted to the test correlation structure. This will usually form smaller intervals that are valid if sample size is large enough. (Note that the implementation in `csranks` package~\cite{chetverikov2024csranks} assumes Normality, and is not suited for small samples).

\subsection{Runtime and scalability analysis}

We analyzed the theoretical runtime of our framework to validate its scalability and confirm that it scales linearly with the number of tasks.

Let $M$ be the number of models, $N$ the number of tasks, and $n$ the number of base values per model-task pair. The end-to-end runtime consists of three main steps:
\begin{enumerate}
    \item \textbf{P-Value Computation}: Computing pairwise paired t-tests for all model combinations on a single task takes $O(nM^2)$. Across all tasks, this requires $O(N n M^2)$.
    \item \textbf{Task-level rank CIs} (Algorithm~\ref{alg:task_ci}): For marginal rank CIs, applying the step-down FWER control requires sorting $M-1$ p-values for each of the $M$ models, yielding $O(M^2 \log M)$ per task. Across all tasks, this step requires $O(N M^2 \log M)$.
    \item \textbf{Leaderboard-level rank PIs} (Algorithm~\ref{alg:quantile}): A straightforward approach would require sorting the per-task bounds, with a runtime of $O(M N \log N)$. However, our implementation leverages optimized selection algorithms (via NumPy and Pandas) for quantile identification, avoiding full sorting and reducing the average runtime to $O(M N)$.
\end{enumerate}

Total average runtime: $O(N n M^2 + N M^2 \log M + M N) \approx O(NM^2n)$.

Because tasks are independent, the runtime for computing task-level rank confidence intervals (CIs) mainly depends on available parallelism (i.e., number of CPUs) rather than the total number of tasks. To assess the computational complexity of Algorithm~\ref{alg:quantile}, we measured execution time for varying numbers of tasks $N$, using moderate correlation ($\rho=0.5$, $block\_size=3$) and no ties. As shown in Figure~\ref{fig:app_runtime_quantile}, runtime increases linearly with $N$, although variability also grows at higher $N$, likely due to memory management overhead.

\begin{figure}[ht]
    \centering
    \includegraphics[width=0.95\linewidth]{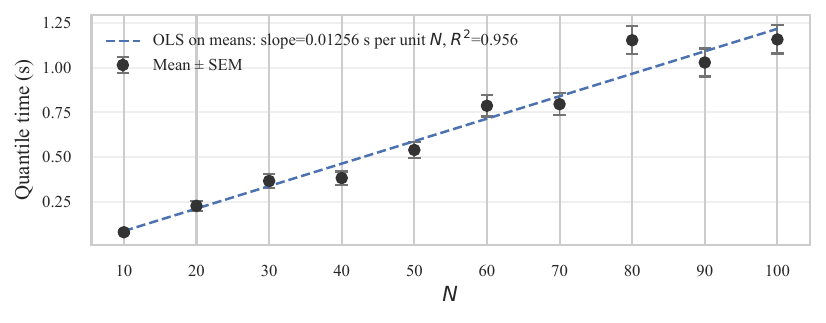}
      \caption{Runtime of the quantile method for leaderboard-level rank PIs increases linearly with $N$.}
    \label{fig:app_runtime_quantile}
\end{figure}

Since the runtime scales linearly with the number of tasks, the framework is efficient and well-suited for large-scale leaderboards. For practical use in such settings, we recommend storing pairwise comparison results (such as t-test p-values) offline. When new tasks or models are introduced, simply update their p-values and recompute the task-level rank CIs and leaderboard-level PIs using these stored results.

\subsection{Reproducibility instructions and computing infrastructure}

All code for our hierarchical ranking methods and the scripts to reproduce the experiments (Python 3.13.5) are publicly available at: \url{https://github.com/BityaNeuhof/leaderboard-rank-intervals}. Synthetic data simulations (see Appendix~\ref{app:synthetic_desc}) were run on a server with 128 GB RAM and 100 CPUs.
\section{Rank CIs for pairwise preferences data}\label{app:pairwise_preferences}

As LLM evaluation progresses toward reflecting human preferences, static leaderboards are increasingly seen as insufficient~\cite{chiang2024chatbot}. Instead of relying on fixed benchmarks, models answer open-ended, real-world questions, and human or LLM judges compare pairs of answers to determine which is better, resulting in pairwise preference data. For each prompt, a pair of models is randomly sampled and compared, rather than evaluating all models on a static set of prompts. Scores are estimated via the Bradley-Terry (BT) model~\cite{bradley1952rank} or its variants~\cite{chiang2024chatbot, ameli2024statistical, son2025arena}, and models are then ranked accordingly. Pairwise data can be bootstrapped to obtain CIs for these scores. Recent work~\cite{wang2024confidence, avelar2026prompt, jung2026defines} has introduced more robust uncertainty assessments for model ranks, such as nonparametric ranking diagrams and prompt-dependent analysis. Task-level rank CIs can similarly be constructed from BT scores using bootstrap resampling or parametric estimation of the covariance between model scores, allowing for more fine-grained insight into rank uncertainty across different prompts. 

Following this shift toward uncertainty-aware and context-dependent evaluation, we applied our task-level rank CI method (Algorithm~\ref{alg:task_ci}) to the Arena Human Preference 140K dataset.~\footnote{\href{https://huggingface.co/datasets/lmarena-ai/arena-human-preference-140k}{Arena human preference 140k Hugging face dataset card.}} The dataset consists of competitions comparing pairs of models ($model\_a, model\_b$) on user conversations. Each competition can result in $model\_a$ or $model\_b$ winning, a tie ($tie$), or both losing ($both\_bad$); both $tie$ and $both\_bad$ are treated as ties in our analysis. We focused our analysis on 45 models with more than 2,000 recorded competitions to ensure robust estimation. We adopted the parametric approach from the Arena-Rank repository,~\footnote{\href{https://github.com/lmarena/arena-rank?ref=arena.ai}{Arena-Rank repository}} which computes BT scores and their associated covariance matrix using a z-distribution. This approach provides a computationally scalable alternative to bootstrapping, matching the efficiency goals of recent tournament-based benchmarks~\cite{son2025arena} while maintaining the statistical depth needed for rank inference under prompt variability. Using these scores and covariances, we applied Algorithm~\ref{alg:task_ci} with Holm's FWER correction ($\alpha = 0.05$) to construct the rank CIs. Estimation was performed using a local implementation adapted from the official Arena-Rank repository, with support for the contextual BT model. Compared to the "rank spread" intervals currently reported by Arena AI,~\footnote{\href{https://arena.ai/blog/ranking-method/}{Arena's blog on ranking method.}} our FWER rank CIs are identical for approximately half of the models and wider for others, as summarized in Table~\ref{tab:app_chatbotarena_results}. Our intervals are guaranteed to include the true rank of each model with probability $1 - \alpha$, providing a more conservative and statistically valid bound than standard spread measures, particularly in high-variance scenarios involving prompt-dependent performance shifts~\cite{avelar2026prompt}.

\begin{table}[ht]
\centering
\caption{Rank CIs comparison}
\label{tab:app_chatbotarena_results}
\resizebox{\textwidth}{!}{%
\begin{tabular}{@{}llllll@{}}
\toprule
Model & Raw rank & Rank spread & FWER rank CI & BT score & BT score CI \\ \midrule
gemini-2.5-pro & 1 & (1, 1) & (1, 1) & 1126.47 & (1120.09, 1132.86) \\
o3-2025-04-16 & 2 & (2, 7) & (2, 7) & 1082.69 & (1076.17, 1089.22) \\
chatgpt-4o-latest-20250326 & 3 & (2, 7) & (2, 7) & 1081.35 & (1074.47, 1088.23) \\
gemini-2.5-pro-preview-05-06 & 4 & (2, 8) & (2, 8) & 1077.86 & (1067.67, 1088.04) \\
deepseek-r1-0528 & 5 & (2, 8) & (2, 8) & 1076.53 & (1069.29, 1083.78) \\
grok-3-preview-02-24 & 6 & (2, 8) & (2, 8) & 1074.11 & (1066.47, 1081.75) \\
llama-4-maverick-03-26-experimental & 7 & (2, 8) & (2, 8) & 1069.55 & (1061.76, 1077.34) \\
gemini-2.5-flash & 8 & (4, 8) & (4, 8) & 1063.81 & (1057.94, 1069.68) \\
qwen3-235b-a22b-no-thinking & 9 & (9, 11) & (9, 11) & 1044.47 & (1038.45, 1050.5) \\
gemini-2.5-flash-preview-04-17 & 10 & (9, 18) & (9, 18) & 1037.17 & (1028.98, 1045.36) \\
kimi-k2-0711-preview & 11 & (9, 19) & (9, 20) & 1033.42 & (1022.11, 1044.74) \\
gpt-4.1-2025-04-14 & 12 & (10, 19) & (9, 20) & 1030.27 & (1023.1, 1037.44) \\
deepseek-v3-0324 & 13 & (10, 19) & (9, 20) & 1029.53 & (1022.22, 1036.83) \\
qwen-max-2025-01-25 & 14 & (10, 20) & (9, 20) & 1028.6 & (1019.99, 1037.21) \\
qwen3-235b-a22b & 15 & (10, 20) & (10, 20) & 1027.72 & (1019.79, 1035.64) \\
mistral-medium-2505 & 16 & (10, 20) & (10, 20) & 1025.82 & (1019.87, 1031.76) \\
o4-mini-2025-04-16 & 17 & (10, 20) & (10, 21) & 1025.68 & (1018.49, 1032.87) \\
gemini-2.5-flash-lite-preview-06-17-thinking & 18 & (10, 22) & (10, 23) & 1024.24 & (1016.27, 1032.21) \\
minimax-m1 & 19 & (11, 24) & (11, 24) & 1016.65 & (1008.83, 1024.47) \\
gemma-3-27b-it & 20 & (14, 24) & (11, 24) & 1015.18 & (1008.43, 1021.92) \\
claude-opus-4-20250514-thinking-16k & 21 & (18, 24) & (17, 25) & 1009.65 & (1002.18, 1017.13) \\
grok-3-mini-beta & 22 & (18, 24) & (17, 26) & 1009.41 & (1001.98, 1016.84) \\
claude-opus-4-20250514 & 23 & (19, 24) & (19, 25) & 1007.92 & (1002.18, 1013.65) \\
grok-3-mini-high & 24 & (19, 29) & (17, 29) & 1005.02 & (994.19, 1015.85) \\
claude-sonnet-4-20250514-thinking-32k & 25 & (24, 31) & (24, 32) & 992.34 & (984.56, 1000.12) \\
gemini-2.0-flash-001 & 26 & (24, 31) & (24, 32) & 991.42 & (983.97, 998.86) \\
qwq-32b & 27 & (24, 32) & (24, 32) & 989.6 & (980.74, 998.47) \\
mistral-small-2506 & 28 & (24, 32) & (21, 32) & 987.95 & (976.29, 999.61) \\
gpt-4.1-mini-2025-04-14 & 29 & (24, 32) & (24, 32) & 987.03 & (979.73, 994.34) \\
qwen3-30b-a3b & 30 & (25, 32) & (25, 32) & 981.07 & (973.4, 988.75) \\
command-a-03-2025 & 31 & (25, 32) & (25, 34) & 977.89 & (971.06, 984.72) \\
claude-sonnet-4-20250514 & 32 & (27, 32) & (25, 35) & 977.04 & (970.66, 983.43) \\
o3-mini & 33 & (33, 35) & (33, 36) & 961.53 & (954.33, 968.73) \\
amazon-nova-experimental-chat-05-14 & 34 & (33, 36) & (32, 38) & 959.25 & (949.2, 969.29) \\
gemma-3n-e4b-it & 35 & (33, 38) & (32, 38) & 957.85 & (946.51, 969.2) \\
llama-4-scout-17b-16e-instruct & 36 & (34, 40) & (34, 40) & 939.81 & (928.81, 950.81) \\
claude-3-7-sonnet-20250219-thinking-32k & 37 & (35, 40) & (35, 40) & 939.67 & (932.63, 946.71) \\
llama-4-maverick-17b-128e-instruct & 38 & (35, 40) & (35, 40) & 939.06 & (930.94, 947.17) \\
claude-3-7-sonnet-20250219 & 39 & (36, 40) & (36, 40) & 933.81 & (926.53, 941.08) \\
claude-3-5-sonnet-20241022 & 40 & (36, 40) & (36, 40) & 932.98 & (925.73, 940.23) \\
llama-3.3-70b-instruct & 41 & (41, 44) & (41, 45) & 914.06 & (905.44, 922.67) \\
amazon.nova-pro-v1:0 & 42 & (41, 44) & (41, 45) & 913.6 & (906.28, 920.92) \\
mistral-small-3.1-24b-instruct-2503 & 43 & (41, 45) & (41, 45) & 910.09 & (899.89, 920.29) \\
claude-3-5-haiku-20241022 & 44 & (41, 45) & (41, 45) & 900.34 & (892.9, 907.79) \\
magistral-medium-2506 & 45 & (43, 45) & (44, 45) & 890.5 & (878.45, 902.55) \\ \bottomrule
\end{tabular}
}
\end{table}

Although our framework can treat prompt categories as separate tasks, the current prompt categorization in the Arena Human Preference 140K dataset is non-exclusive, with prompts often assigned to multiple categories. For this reason, our analysis focuses on the overall preference distribution. If future datasets provide distinct, non-overlapping prompt categories, task-level rank CIs (or alternative valid methods) can be computed for each category and then aggregated to obtain leaderboard-level rank PIs.
\section{Proof of Theorem~\ref{thm:leaderboard}}\label{app:proof}

Formally, we can define a task $t_b$ to consist of a triplet $(\truescore, \mathbf{\basematrix} ,\Sigma)^b$, where $\truescore^b \in R^M$ is the true performance score vector,
$\mathbf{\basematrix}^b$ is the observed performance matrix, and $\Sigma$ is the correlation between models. The true ranks for the task $r^b$ can be derived from $\truescore^b$.

\begin{theorem*}[~\ref{thm:leaderboard}]
Assume the observed tasks $t^1, \ldots, t^N$ are independently sampled from the distribution $\mathcal{P}$. Construct the collection $\mathcal{I}_{M \times N}$ of task-level rank CIs, each with marginal coverage $1 - \alpha_{tsk}$.
Now, consider a new independent task $t^* \sim \mathcal{P}$, with true ranks $r_1^*, \ldots, r_M^*$.
    The leaderboard rank intervals $([L_1, U_1], \ldots, [L_M, U_M])$ produced by Algorithm~\ref{alg:quantile} achieve marginal task coverage of $1 - (\alpha_{ldb} + \alpha_{tsk})$. That is, for each model $c_j$,
    \begin{equation}\label{eq:app_thm_leaderboard}
        \begin{aligned}
            \prob \big( r_j^* \subseteq [L_j, U_j]\big)
            \geq 1 - (\alpha_{ldb} + \alpha_{tsk}).
        \end{aligned}
    \end{equation}
\end{theorem*}

First, for the new task $t^*$, we estimate a task-level rank CI $[L^*_j, U^*_j]$, with marginal coverage rate $1-\alpha_{tsk}$. 
The interval $[L^*_j, U^*_j]$ has two key properties:
\begin{enumerate}
    \item It covers the true rank $r^*_j$ with high probability $\prob \big( r^*_j \subseteq [L^*_j, U^*_j]\big) \geq 1-\alpha_{tsk}$.
    \item It is exchangeable with the task-level rank CIs $[L^1_j, U^1_j],\ldots , [L^N_j, U^N_j]$ because all intervals are derived from the same estimation method applied to independent and identically distributed tasks.
\end{enumerate}

Due to property (1), it suffices to show that $\prob \big([L^*_j, U^*_j] \subseteq [L_j, U_j] \big) \geq 1-\alpha_{ldb}$. By applying the union bound, we then obtain the desired coverage for the true rank.

Following a similar argument to~\citet{lei2018distribution}, we can verify that the quantiles defined in Algorithm~\ref{alg:quantile} provide sufficient coverage probabilities.
We bound the probability of a lower-side coverage error, that is, $\prob(L^*_j < L_j)$.

Because $L^*_j, L^1_j, \dots, L^N_j$ are exchangeable, any ordering of these $N+1$ integers is equally likely.

Specifically, let $L[1] \leq L[2] \leq \ldots \leq L[N]$ denote the order statistics of $L_1, \ldots, L_N$. Then,
$\prob(L^*_j < L[k]) \leq \frac{k}{N+1}$. (Equality holds if there are no ties.)
Substituting $L_j = L[k_l]$, we obtain:

\begin{equation}\label{eq:app_error_bound}
    \begin{aligned}
        \prob(L^*_j<L_j) &\leq \frac{k_l}{N+1}
\leq \left\lfloor (N+1)\frac{\alpha_{ldb}}{2} \right\rfloor \frac{1}{N+1} \leq \frac{\alpha_{ldb}}{2}.
    \end{aligned}
\end{equation}

Applying the same argument to the upper bounds $U^*_j$ and $U_j$ results in an additional coverage error of at most $\frac{\alpha_{ldb}}{2}$.

In summary, a coverage error can occur if (1) $r^*_j$ is not covered by $[L^*_j, U^*_j]$ (event $E_1$), (2) the new lower bound falls below the aggregated lower bound ($L^*_j < L_j$, event $E_2$), or (3) the new upper bound exceeds the aggregated upper bound ($U^*_j > U_j$, event $E_3$).

By union bound, we can bound the probability of the coverage error by:
\begin{equation}\label{eq:app_union_bound}
    \begin{aligned}
        \prob (E_1 \cup E_2 \cup E_3) \leq 
        \alpha_{tsk}+ 2\cdot \frac{\alpha_{ldb}}{2}
        = \alpha_{tsk} + \alpha_{ldb}
    \end{aligned}
\end{equation}
Meaning that the coverage probability is at least $1-(\alpha_{ldb} + \alpha_{tsk})$.
\section{Rank interval interpretation}\label{app:interp}

\subsection{Quantifying uncertainty in rankings}
Sections~\ref{sec:task_ranking} and~\ref{sec:leaderboard_ranking} introduced the two main components of our hierarchical ranking framework: Task-level rank CIs and leaderboard-level rank PIs. At each level, these rank intervals quantify the uncertainty in the estimated model ordering. The width of each interval indicates how precisely we can rank the models: if there is no information about performance, all models will be ranked between $[1, M]$. As more information becomes available, we can statistically distinguish between models, resulting in narrower rank CIs. 

Reducing variance at the task level, for example, by increasing the number of base values per model–task pair, is expected to narrow the interval width. These intervals help interpret relationships both within and across models, highlighting where rankings are stable, uncertain, or models are effectively tied. However, leaderboard PIs should not be mistaken for CIs. A CI estimates a model's average performance, with its width reflecting uncertainty in that estimate. In contrast, a PI describes how much a model’s rank varies across different tasks. Collecting more data narrows the CI by reducing uncertainty, but does not necessarily narrow the PI, since the spread of ranks across tasks may remain the same.

\subsection{Task selection strategies}
At the task level, rank CIs allow direct comparison of models within each task and help assess the average interval width. This informs which tasks should be included in the leaderboard aggregation. For instance, if models perform similarly on a task, resulting in wide or overlapping intervals, it may make sense to exclude that task to focus on those that better distinguish model performance. Strategies for reducing the number of tasks include:
\begin{enumerate}
\item Selecting a subset of tasks based on domain knowledge, which can be done at any stage without affecting the validity of the task-level rank CIs.
\item Filtering tasks whose rank intervals exceed a predefined width threshold.
\item Grouping similar tasks, such as clustering them by their top-k models, and then computing task-level CIs for each cluster.
\end{enumerate}

In Section~\ref{sec:tabarena}, Figure~\ref{fig:tabarena_task_dist}, we show that some models are better on regression tasks, while other outperform on multiclass classification. While comparing all models across all datasets is occasionally of interest, practitioners often focus on more targeted analyses, for example, evaluating tree-based models on regression datasets alone. We apply our ranking framework to a subset of the TabArena data comprising 18 tree-based models evaluated on 13 regression datasets. The results are shown in Figure \ref{fig:tabarena_regression_ranks}.

\begin{figure}[ht]
    \centering
    \includegraphics[width=0.95\columnwidth]{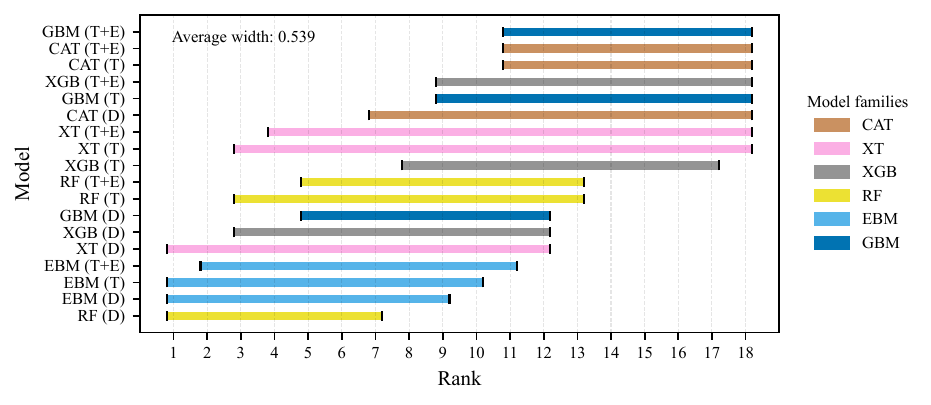}
    \caption{Leaderboard-level rank prediction intervals for 18 tree-based models on 13 regression datasets. Three model families clearly dominate the top: CAT, GBM, and XGB.}
    \label{fig:tabarena_regression_ranks}
\end{figure}

Based on the results, we can conclude that three families of models are consistently ranked at the top: CAT, GBM, and XGB, with little difference between the `tuned' and `tuned + ensemble' variations of models within each family. This is an example for the first strategy, selecting a subset of tasks of interest.

The second and third strategies are data-driven, so they require methods like sample splitting or $\alpha$-splitting to separate task selection from inference. This ensures that the task-level miscoverage rate remains controlled at $\alpha_{tsk}$~\cite{anderson2017split}.

At the leaderboard level, the width of the rank PIs depends on three factors: the width of the task-level rank CIs, the variability of model rank between tasks, and the leaderboard miscoverage rate $\alpha_{ldb}$. The leaderboard PI can also serve as an implicit task-selection mechanism. For example, setting $\alpha_{ldb}=0.5$ results in intervals that focus on tasks within the interquartile range (IQR) of the task-level rank CIs. However, this approach may yield a different set of tasks for each model.

\section{Experiments - simulation details}\label{app:synthetic_desc}

This section describes the synthetic data generation process and simulation parameters used to evaluate our hierarchical ranking framework. By generating synthetic performance data with known true ranks for each task, we can validate coverage properties and directly compare different ranking aggregation methods under controlled conditions.

\subsection{Simulation parameters}

\paragraph{Model and task parameters}
We simulate 1000 tasks, each involving a set of models characterized by a vector of true performance scores and a covariance matrix. The number of models is set as $M \in \{10, 30\}$, with $\truescore = [1, \sqrt{2}, \ldots, \sqrt{M}]$. The covariance matrix $\Sigma$ encodes both the variance of each model ($\sigma \in \{0.3, 0.7, 1.2\}$) and the pairwise correlations between models ($\rho \in \{0, 0.2, 0.5, 0.8\}$), structured in blocks of size $block\_size \in \{2, 3, 5\}$.

We sample the true performance matrix from a multivariate normal distribution $N(\truescore, \Sigma)$. Each column in that matrix, $\truescore^b$, represents the true performance scores of all models on task $t_b$. To simulate ties within a task, we select a proportion of values from the true performance matrix (as specified by $tie_{proportion}$) and replace them with the mean of those values.

We define $\Sigma_{task}$ by setting all zero-entries of $\Sigma$ to $0.1$. This additional correlation factor represents the task-dependent correlations between models. 
We then sample observed performance matrices $\mathbf{\basematrix}_b \in \mathbb{R}^{n_{base} \times M}$ (base values for task $t_b$), with two different sample sizes: $n_{base} = 10$ (to mimic cross-validation folds) and $n_{base} = 200$ (to represent data observations).

\paragraph{Ranking methods parameters}

At the task level, we construct marginal rank CIs with Algorithm~\ref{alg:task_ci} as described in Section±\ref{sec:task_ranking} and in Appendix~\ref{app:implementation_details}, using $\alpha_{tsk} \in \{0.02, 0.05, 0.1\}$.

At the leaderboard-level, we repeatedly ($n\_repetitions=100$) sample $N \in \{20, 60\}$ tasks, without replacement from the 1000 simulated tasks. We sample 100 additional task-level rank CIs (distinct from the $N$ samples previously, but identically distributed) to represent unseen tasks.

For each model $c_j$, we aggregate the task-level CIs into leaderboard-level rank PIs,  using two approaches: (1) a naive union baseline, which takes the minimum lower and maximum upper bounds across tasks; and (2) the quantile merge method (Algorithm~\ref{alg:quantile}) with $\alpha_{ldb} \in \{0.3, 0.5\}$.

\subsection{Evaluation}

\paragraph{Average normalized width}
We compute the normalized width of both task-level and leaderboard-level interval.
For each task $t_b$, we compute an average over all models:
\begin{equation}\label{eq:app_tsk_width}
    \frac{1}{M} \sum_{j=1}^M \frac{U_j^b - L_j^b}{M-1}.
\end{equation}

At the leaderboard level, we average the rank PIs of all models:
\begin{equation}\label{eq:app_ldb_width}
    \frac{1}{M} \sum_{j=1}^M \frac{U_j - L_j}{M-1}.
\end{equation}

\paragraph{Task-level rank CIs coverage}
For each task $t_b$, coverage is defined as the proportion of models for which the true rank falls within its corresponding interval:
\begin{equation}\label{eq:app_tsk_cover}
    \frac{1}{M} \sum_{j=1}^M \big(r^b_j \in [L^b_j, U^b_j]\big).
\end{equation}
A valid set of task-level rank CIs has coverage of $1 - \alpha_{tsk}$. As long as tasks are marginally covered, a lower coverage is better.

\paragraph{Leaderboard-level rank PIs coverage}
To evaluate leaderboard-level rank PIs, we calculate for each model $c_j$ the proportion of $N$ sampled tasks where the true rank is covered:
\begin{equation}\label{eq:app_ldb_cover_model}
    coverage_j = \frac{1}{N} \sum_{b=1}^N \big(r^b_j \in [L_j, U_j]\big).
\end{equation}
Then, we summarize this per-model coverage measure as:
\begin{itemize}
    \item A binary indicator of whether all models achieve at least $1 - \alpha_{ldb}$ coverage:
    \begin{equation}\label{eq:app_ldb_cover_all}
        1 \text{ if } \{ coverage_j \geq (1 - \alpha_{ldb}) \forall j \}.  
    \end{equation}
    \item Average coverage:
    \begin{equation}\label{eq:app_ldb_cover_avg}
        \frac{1}{M} \sum_{j=1}^M coverage_j.
    \end{equation}
\end{itemize}

Finally, we compute and report the average coverage on the unseen tasks, $t_{N+1}, \ldots, t_{N+100}$, to evaluate PI coverage:
\begin{equation}\label{eq:app_ldb_cover_unseen}
    \frac{1}{M}\sum_{j=1}^M \left(\frac{1}{100}\sum_{b^*=N+1}^{N+100} (r_j^{b^*} \in [L_j, U_j])\right).
\end{equation}

\subsection{Task-level simulations}

In addition to the full framework simulations, we also analyze Algorithm~\ref{alg:task_ci}'s performance on a single task, and compare it to a bootstrap baseline as described by Algorithm~\ref{alg:task_bootstrap}.

\begin{algorithm}[ht]
  \caption{Bootstrap Rank Aggregation}
  \label{alg:task_bootstrap}
  \begin{algorithmic}[1]
    \STATE {\bfseries Input:} \\
    \qquad Observed performance matrix $\basematrix$; \\
    \qquad Number of bootstrap repetitions $B$; \\
    \qquad Task-level miscoverage rate $\alpha_{tsk}$.

    \STATE $k_l \gets \left\lfloor B\left(\frac{\alpha_{tsk}}{2}\right)\right\rfloor$; $k_u \gets \left\lceil B\left(1-\frac{\alpha_{tsk}}{2}\right) \right\rceil$
    \STATE $ranks = []$ \COMMENT{Initialized to store B ranks.}
    \FOR{$i=1$ {\bfseries to} $B$}
        \STATE Sample $\basematrix$ with replacement.
        \STATE Compute $\score(i)$ \COMMENT{The observed scores for all models, based on the $i$'th bootstrap sample.}
        \STATE $ranks[i] \gets sorted(\score(i))$
        \COMMENT{Ranks are assigned accounting for ties and indexed by model.}
     \ENDFOR   
        
    \FOR{$j=1$ {\bfseries to} $M$}
        \STATE $ranks_j \gets \{ranks[i]_j : i = 1, ..., B\}$ \COMMENT{All ranks of the $j$'th model across $B$ bootstraps.}
        \STATE $L_j^{bootstrap} \gets ranks_j[k_l]$; $U_j^{bootstrap} \gets ranks_j[k_u]$
    \ENDFOR                
    \STATE Return $\big\{[L_1^{bootstrap}, U_1^{bootstrap}], \ldots, [L_M^{bootstrap}, U_M^{bootstrap}]\big\}$
  \end{algorithmic}
\end{algorithm}

For each parameter configuration, we repeatedly ($n\_repetitions=500$) sample an observed performance matrix and construct both FWER (Algorithm~\ref{alg:task_ci}) and bootstrap rank CIs. We then compute the normalized width (Equation~\ref{eq:app_tsk_width}) and coverage (Equation~\ref{eq:app_tsk_cover}) of the intervals.
\section{Experiments - additional results}\label{app:synthetic_add}

\subsection{Comparison of quantile and union leaderboard-level rank PIs}
We present additional results in Tables~\ref{tab:no_ties_no_corr}-\ref{tab:ties_corr}. We can see that our method keeps the required coverage rate with far greater efficiency compared to the union bound.

\expandafter\providecommand\csname NoTiesNoCorr\endcsname{%
\begin{tabular}{@{}cccccccccccc@{}}
\toprule
\multirow{3}{*}{\# Models} & \multirow{3}{*}{\# Tasks} & \multirow{3}{*}{\# Base values} & \multirow{3}{*}{$\sigma$} & \multicolumn{4}{c}{$\alpha_{ldb}=0.3$} & \multicolumn{4}{c}{$\alpha_{ldb}=0.5$} \\
\cmidrule(lr){5-8}\cmidrule(lr){9-12}
 &  &  &  & \multicolumn{2}{c}{Width mean $\pm$ SD} & \multicolumn{2}{c}{Coverage mean $\pm$ SD} & \multicolumn{2}{c}{Width mean $\pm$ SD} & \multicolumn{2}{c}{Coverage mean $\pm$ SD} \\
\cmidrule(lr){5-6}\cmidrule(lr){7-8}\cmidrule(lr){9-10}\cmidrule(lr){11-12}
 &  &  &  & Union & Quantile & Union & Quantile & Union & Quantile & Union & Quantile \\
\midrule
10 & 20 & 10 & 0.3 & 0.72 $\pm$ 0.03 & 0.57 $\pm$ 0.04 & 1.00 $\pm$ 0.00 & 0.99 $\pm$ 0.00 & 0.72 $\pm$ 0.03 & 0.48 $\pm$ 0.04 & 1.00 $\pm$ 0.00 & 0.98 $\pm$ 0.01 \\
10 & 20 & 10 & 0.7 & 0.98 $\pm$ 0.01 & 0.92 $\pm$ 0.03 & 1.00 $\pm$ 0.00 & 1.00 $\pm$ 0.00 & 0.98 $\pm$ 0.01 & 0.84 $\pm$ 0.04 & 1.00 $\pm$ 0.00 & 0.98 $\pm$ 0.01 \\
10 & 20 & 10 & 1.2 & 1.00 $\pm$ 0.00 & 0.99 $\pm$ 0.01 & 1.00 $\pm$ 0.00 & 1.00 $\pm$ 0.00 & 1.00 $\pm$ 0.00 & 0.95 $\pm$ 0.03 & 1.00 $\pm$ 0.00 & 0.99 $\pm$ 0.01 \\
10 & 20 & 200 & 0.3 & 0.49 $\pm$ 0.01 & 0.32 $\pm$ 0.01 & 0.98 $\pm$ 0.00 & 0.91 $\pm$ 0.01 & 0.49 $\pm$ 0.01 & 0.22 $\pm$ 0.01 & 0.98 $\pm$ 0.00 & 0.82 $\pm$ 0.01 \\
10 & 20 & 200 & 0.7 & 0.83 $\pm$ 0.00 & 0.62 $\pm$ 0.01 & 0.98 $\pm$ 0.00 & 0.90 $\pm$ 0.01 & 0.83 $\pm$ 0.00 & 0.46 $\pm$ 0.01 & 0.98 $\pm$ 0.00 & 0.79 $\pm$ 0.01 \\
10 & 20 & 200 & 1.2 & 0.94 $\pm$ 0.00 & 0.77 $\pm$ 0.01 & 0.98 $\pm$ 0.00 & 0.90 $\pm$ 0.01 & 0.94 $\pm$ 0.00 & 0.60 $\pm$ 0.01 & 0.98 $\pm$ 0.00 & 0.78 $\pm$ 0.01 \\
10 & 60 & 10 & 0.3 & 0.80 $\pm$ 0.03 & 0.56 $\pm$ 0.04 & 1.00 $\pm$ 0.00 & 0.99 $\pm$ 0.00 & 0.80 $\pm$ 0.03 & 0.47 $\pm$ 0.04 & 1.00 $\pm$ 0.00 & 0.98 $\pm$ 0.01 \\
10 & 60 & 10 & 0.7 & 1.00 $\pm$ 0.00 & 0.91 $\pm$ 0.03 & 1.00 $\pm$ 0.00 & 1.00 $\pm$ 0.00 & 1.00 $\pm$ 0.00 & 0.83 $\pm$ 0.05 & 1.00 $\pm$ 0.00 & 0.99 $\pm$ 0.01 \\
10 & 60 & 10 & 1.2 & 1.00 $\pm$ 0.00 & 0.99 $\pm$ 0.01 & 1.00 $\pm$ 0.00 & 1.00 $\pm$ 0.00 & 1.00 $\pm$ 0.00 & 0.96 $\pm$ 0.02 & 1.00 $\pm$ 0.00 & 0.99 $\pm$ 0.01 \\
10 & 60 & 200 & 0.3 & 0.59 $\pm$ 0.01 & 0.31 $\pm$ 0.01 & 0.99 $\pm$ 0.00 & 0.91 $\pm$ 0.01 & 0.59 $\pm$ 0.01 & 0.21 $\pm$ 0.01 & 0.99 $\pm$ 0.00 & 0.82 $\pm$ 0.01 \\
10 & 60 & 200 & 0.7 & 0.93 $\pm$ 0.00 & 0.61 $\pm$ 0.01 & 1.00 $\pm$ 0.00 & 0.90 $\pm$ 0.01 & 0.93 $\pm$ 0.00 & 0.45 $\pm$ 0.01 & 1.00 $\pm$ 0.00 & 0.79 $\pm$ 0.01 \\
10 & 60 & 200 & 1.2 & 0.99 $\pm$ 0.00 & 0.76 $\pm$ 0.01 & 1.00 $\pm$ 0.00 & 0.91 $\pm$ 0.01 & 0.99 $\pm$ 0.00 & 0.59 $\pm$ 0.01 & 1.00 $\pm$ 0.00 & 0.78 $\pm$ 0.01 \\
30 & 20 & 10 & 0.3 & 0.48 $\pm$ 0.03 & 0.38 $\pm$ 0.03 & 1.00 $\pm$ 0.00 & 0.99 $\pm$ 0.00 & 0.48 $\pm$ 0.03 & 0.33 $\pm$ 0.03 & 1.00 $\pm$ 0.00 & 0.99 $\pm$ 0.01 \\
30 & 20 & 10 & 0.7 & 0.84 $\pm$ 0.03 & 0.73 $\pm$ 0.04 & 1.00 $\pm$ 0.00 & 1.00 $\pm$ 0.00 & 0.84 $\pm$ 0.03 & 0.66 $\pm$ 0.05 & 1.00 $\pm$ 0.00 & 0.99 $\pm$ 0.01 \\
30 & 20 & 10 & 1.2 & 0.97 $\pm$ 0.01 & 0.92 $\pm$ 0.03 & 1.00 $\pm$ 0.00 & 1.00 $\pm$ 0.00 & 0.97 $\pm$ 0.01 & 0.86 $\pm$ 0.04 & 1.00 $\pm$ 0.00 & 0.99 $\pm$ 0.01 \\
30 & 20 & 200 & 0.3 & 0.29 $\pm$ 0.00 & 0.19 $\pm$ 0.00 & 0.97 $\pm$ 0.00 & 0.88 $\pm$ 0.01 & 0.29 $\pm$ 0.00 & 0.14 $\pm$ 0.00 & 0.97 $\pm$ 0.00 & 0.78 $\pm$ 0.01 \\
30 & 20 & 200 & 0.7 & 0.58 $\pm$ 0.01 & 0.41 $\pm$ 0.01 & 0.97 $\pm$ 0.00 & 0.87 $\pm$ 0.01 & 0.58 $\pm$ 0.01 & 0.30 $\pm$ 0.01 & 0.97 $\pm$ 0.00 & 0.74 $\pm$ 0.01 \\
30 & 20 & 200 & 1.2 & 0.78 $\pm$ 0.01 & 0.59 $\pm$ 0.01 & 0.97 $\pm$ 0.00 & 0.87 $\pm$ 0.01 & 0.78 $\pm$ 0.01 & 0.45 $\pm$ 0.01 & 0.97 $\pm$ 0.00 & 0.74 $\pm$ 0.01 \\
30 & 60 & 10 & 0.3 & 0.54 $\pm$ 0.02 & 0.37 $\pm$ 0.03 & 1.00 $\pm$ 0.00 & 1.00 $\pm$ 0.00 & 0.54 $\pm$ 0.02 & 0.32 $\pm$ 0.03 & 1.00 $\pm$ 0.00 & 0.99 $\pm$ 0.01 \\
30 & 60 & 10 & 0.7 & 0.90 $\pm$ 0.02 & 0.73 $\pm$ 0.04 & 1.00 $\pm$ 0.00 & 1.00 $\pm$ 0.00 & 0.90 $\pm$ 0.02 & 0.65 $\pm$ 0.05 & 1.00 $\pm$ 0.00 & 0.99 $\pm$ 0.01 \\
30 & 60 & 10 & 1.2 & 0.99 $\pm$ 0.00 & 0.92 $\pm$ 0.03 & 1.00 $\pm$ 0.00 & 1.00 $\pm$ 0.00 & 0.99 $\pm$ 0.00 & 0.86 $\pm$ 0.04 & 1.00 $\pm$ 0.00 & 0.99 $\pm$ 0.00 \\
30 & 60 & 200 & 0.3 & 0.35 $\pm$ 0.00 & 0.18 $\pm$ 0.00 & 0.99 $\pm$ 0.00 & 0.88 $\pm$ 0.01 & 0.35 $\pm$ 0.00 & 0.13 $\pm$ 0.00 & 0.99 $\pm$ 0.00 & 0.78 $\pm$ 0.01 \\
30 & 60 & 200 & 0.7 & 0.68 $\pm$ 0.00 & 0.39 $\pm$ 0.01 & 0.99 $\pm$ 0.00 & 0.87 $\pm$ 0.01 & 0.68 $\pm$ 0.00 & 0.29 $\pm$ 0.01 & 0.99 $\pm$ 0.00 & 0.74 $\pm$ 0.01 \\
30 & 60 & 200 & 1.2 & 0.88 $\pm$ 0.00 & 0.57 $\pm$ 0.01 & 0.99 $\pm$ 0.00 & 0.87 $\pm$ 0.01 & 0.88 $\pm$ 0.00 & 0.43 $\pm$ 0.01 & 0.99 $\pm$ 0.00 & 0.73 $\pm$ 0.01 \\
\bottomrule
\end{tabular}
}

\expandafter\providecommand\csname TiesNoCorr\endcsname{%
\begin{tabular}{@{}cccccccccccc@{}}
\toprule
\multirow{3}{*}{\# Models} & \multirow{3}{*}{\# Tasks} & \multirow{3}{*}{\# Base values} & \multirow{3}{*}{$\sigma$} & \multicolumn{4}{c}{$\alpha_{ldb}=0.3$} & \multicolumn{4}{c}{$\alpha_{ldb}=0.5$} \\
\cmidrule(lr){5-8}\cmidrule(lr){9-12}
 &  &  &  & \multicolumn{2}{c}{Width mean $\pm$ SD} & \multicolumn{2}{c}{Coverage mean $\pm$ SD} & \multicolumn{2}{c}{Width mean $\pm$ SD} & \multicolumn{2}{c}{Coverage mean $\pm$ SD} \\
\cmidrule(lr){5-6}\cmidrule(lr){7-8}\cmidrule(lr){9-10}\cmidrule(lr){11-12}
 &  &  &  & Union & Quantile & Union & Quantile & Union & Quantile & Union & Quantile \\
\midrule
10 & 20 & 10 & 0.3 & 0.83 $\pm$ 0.11 & 0.71 $\pm$ 0.15 & 0.99 $\pm$ 0.01 & 0.95 $\pm$ 0.03 & 0.83 $\pm$ 0.11 & 0.62 $\pm$ 0.16 & 0.99 $\pm$ 0.01 & 0.90 $\pm$ 0.05 \\
10 & 20 & 10 & 0.7 & 0.99 $\pm$ 0.01 & 0.95 $\pm$ 0.04 & 1.00 $\pm$ 0.00 & 0.99 $\pm$ 0.01 & 0.99 $\pm$ 0.01 & 0.89 $\pm$ 0.06 & 1.00 $\pm$ 0.00 & 0.96 $\pm$ 0.02 \\
10 & 20 & 10 & 1.2 & 1.00 $\pm$ 0.00 & 0.99 $\pm$ 0.01 & 1.00 $\pm$ 0.00 & 1.00 $\pm$ 0.00 & 1.00 $\pm$ 0.00 & 0.97 $\pm$ 0.03 & 1.00 $\pm$ 0.00 & 0.98 $\pm$ 0.01 \\
10 & 20 & 200 & 0.3 & 0.73 $\pm$ 0.20 & 0.60 $\pm$ 0.25 & 0.98 $\pm$ 0.01 & 0.91 $\pm$ 0.02 & 0.73 $\pm$ 0.20 & 0.51 $\pm$ 0.26 & 0.98 $\pm$ 0.01 & 0.81 $\pm$ 0.04 \\
10 & 20 & 200 & 0.7 & 0.91 $\pm$ 0.07 & 0.78 $\pm$ 0.15 & 0.98 $\pm$ 0.01 & 0.91 $\pm$ 0.03 & 0.91 $\pm$ 0.07 & 0.67 $\pm$ 0.20 & 0.98 $\pm$ 0.01 & 0.80 $\pm$ 0.05 \\
10 & 20 & 200 & 1.2 & 0.97 $\pm$ 0.03 & 0.87 $\pm$ 0.09 & 0.99 $\pm$ 0.01 & 0.92 $\pm$ 0.04 & 0.97 $\pm$ 0.03 & 0.76 $\pm$ 0.15 & 0.99 $\pm$ 0.01 & 0.80 $\pm$ 0.06 \\
10 & 60 & 10 & 0.3 & 0.88 $\pm$ 0.08 & 0.70 $\pm$ 0.16 & 1.00 $\pm$ 0.00 & 0.95 $\pm$ 0.03 & 0.88 $\pm$ 0.08 & 0.61 $\pm$ 0.16 & 1.00 $\pm$ 0.00 & 0.90 $\pm$ 0.06 \\
10 & 60 & 10 & 0.7 & 1.00 $\pm$ 0.00 & 0.94 $\pm$ 0.04 & 1.00 $\pm$ 0.00 & 0.99 $\pm$ 0.01 & 1.00 $\pm$ 0.00 & 0.88 $\pm$ 0.07 & 1.00 $\pm$ 0.00 & 0.97 $\pm$ 0.02 \\
10 & 60 & 10 & 1.2 & 1.00 $\pm$ 0.00 & 0.99 $\pm$ 0.01 & 1.00 $\pm$ 0.00 & 1.00 $\pm$ 0.00 & 1.00 $\pm$ 0.00 & 0.97 $\pm$ 0.02 & 1.00 $\pm$ 0.00 & 0.99 $\pm$ 0.01 \\
10 & 60 & 200 & 0.3 & 0.80 $\pm$ 0.17 & 0.59 $\pm$ 0.26 & 0.99 $\pm$ 0.00 & 0.91 $\pm$ 0.02 & 0.80 $\pm$ 0.17 & 0.50 $\pm$ 0.26 & 0.99 $\pm$ 0.00 & 0.81 $\pm$ 0.04 \\
10 & 60 & 200 & 0.7 & 0.96 $\pm$ 0.03 & 0.77 $\pm$ 0.16 & 1.00 $\pm$ 0.00 & 0.91 $\pm$ 0.04 & 0.96 $\pm$ 0.03 & 0.66 $\pm$ 0.20 & 1.00 $\pm$ 0.00 & 0.80 $\pm$ 0.05 \\
10 & 60 & 200 & 1.2 & 0.99 $\pm$ 0.00 & 0.86 $\pm$ 0.10 & 1.00 $\pm$ 0.00 & 0.93 $\pm$ 0.04 & 0.99 $\pm$ 0.00 & 0.75 $\pm$ 0.15 & 1.00 $\pm$ 0.00 & 0.80 $\pm$ 0.07 \\
30 & 20 & 10 & 0.3 & 0.66 $\pm$ 0.18 & 0.57 $\pm$ 0.19 & 0.98 $\pm$ 0.01 & 0.93 $\pm$ 0.05 & 0.66 $\pm$ 0.18 & 0.52 $\pm$ 0.20 & 0.98 $\pm$ 0.01 & 0.87 $\pm$ 0.08 \\
30 & 20 & 10 & 0.7 & 0.90 $\pm$ 0.06 & 0.80 $\pm$ 0.08 & 0.99 $\pm$ 0.01 & 0.95 $\pm$ 0.04 & 0.90 $\pm$ 0.06 & 0.74 $\pm$ 0.10 & 0.99 $\pm$ 0.01 & 0.90 $\pm$ 0.09 \\
30 & 20 & 10 & 1.2 & 0.98 $\pm$ 0.01 & 0.94 $\pm$ 0.03 & 1.00 $\pm$ 0.00 & 0.98 $\pm$ 0.02 & 0.98 $\pm$ 0.01 & 0.89 $\pm$ 0.04 & 1.00 $\pm$ 0.00 & 0.95 $\pm$ 0.05 \\
30 & 20 & 200 & 0.3 & 0.58 $\pm$ 0.25 & 0.49 $\pm$ 0.26 & 0.96 $\pm$ 0.01 & 0.87 $\pm$ 0.02 & 0.58 $\pm$ 0.25 & 0.43 $\pm$ 0.27 & 0.96 $\pm$ 0.01 & 0.76 $\pm$ 0.04 \\
30 & 20 & 200 & 0.7 & 0.77 $\pm$ 0.17 & 0.64 $\pm$ 0.21 & 0.96 $\pm$ 0.01 & 0.85 $\pm$ 0.02 & 0.77 $\pm$ 0.17 & 0.55 $\pm$ 0.23 & 0.96 $\pm$ 0.01 & 0.70 $\pm$ 0.03 \\
30 & 20 & 200 & 1.2 & 0.87 $\pm$ 0.09 & 0.75 $\pm$ 0.15 & 0.96 $\pm$ 0.01 & 0.85 $\pm$ 0.02 & 0.87 $\pm$ 0.09 & 0.65 $\pm$ 0.19 & 0.96 $\pm$ 0.01 & 0.70 $\pm$ 0.03 \\
30 & 60 & 10 & 0.3 & 0.72 $\pm$ 0.17 & 0.56 $\pm$ 0.19 & 0.99 $\pm$ 0.00 & 0.93 $\pm$ 0.05 & 0.72 $\pm$ 0.17 & 0.51 $\pm$ 0.20 & 0.99 $\pm$ 0.00 & 0.88 $\pm$ 0.08 \\
30 & 60 & 10 & 0.7 & 0.94 $\pm$ 0.04 & 0.80 $\pm$ 0.09 & 1.00 $\pm$ 0.00 & 0.96 $\pm$ 0.04 & 0.94 $\pm$ 0.04 & 0.73 $\pm$ 0.10 & 1.00 $\pm$ 0.00 & 0.90 $\pm$ 0.09 \\
30 & 60 & 10 & 1.2 & 1.00 $\pm$ 0.00 & 0.94 $\pm$ 0.03 & 1.00 $\pm$ 0.00 & 0.98 $\pm$ 0.02 & 1.00 $\pm$ 0.00 & 0.89 $\pm$ 0.05 & 1.00 $\pm$ 0.00 & 0.96 $\pm$ 0.05 \\
30 & 60 & 200 & 0.3 & 0.64 $\pm$ 0.24 & 0.48 $\pm$ 0.26 & 0.99 $\pm$ 0.00 & 0.87 $\pm$ 0.02 & 0.64 $\pm$ 0.24 & 0.43 $\pm$ 0.27 & 0.99 $\pm$ 0.00 & 0.76 $\pm$ 0.05 \\
30 & 60 & 200 & 0.7 & 0.84 $\pm$ 0.13 & 0.63 $\pm$ 0.22 & 0.99 $\pm$ 0.00 & 0.85 $\pm$ 0.02 & 0.84 $\pm$ 0.13 & 0.54 $\pm$ 0.23 & 0.99 $\pm$ 0.00 & 0.70 $\pm$ 0.03 \\
30 & 60 & 200 & 1.2 & 0.93 $\pm$ 0.05 & 0.74 $\pm$ 0.15 & 0.99 $\pm$ 0.00 & 0.85 $\pm$ 0.02 & 0.93 $\pm$ 0.05 & 0.64 $\pm$ 0.19 & 0.99 $\pm$ 0.00 & 0.69 $\pm$ 0.03 \\
\bottomrule
\end{tabular}
}

\expandafter\providecommand\csname NoTiesCorr\endcsname{%
\begin{tabular}{@{}cccccccccccc@{}}
\toprule
\multirow{3}{*}{\# Models} & \multirow{3}{*}{\# Tasks} & \multirow{3}{*}{\# Base values} & \multirow{3}{*}{$\sigma$} & \multicolumn{4}{c}{$\alpha_{ldb}=0.3$} & \multicolumn{4}{c}{$\alpha_{ldb}=0.5$} \\
\cmidrule(lr){5-8}\cmidrule(lr){9-12}
 &  &  &  & \multicolumn{2}{c}{Width mean $\pm$ SD} & \multicolumn{2}{c}{Coverage mean $\pm$ SD} & \multicolumn{2}{c}{Width mean $\pm$ SD} & \multicolumn{2}{c}{Coverage mean $\pm$ SD} \\
\cmidrule(lr){5-6}\cmidrule(lr){7-8}\cmidrule(lr){9-10}\cmidrule(lr){11-12}
 &  &  &  & Union & Quantile & Union & Quantile & Union & Quantile & Union & Quantile \\
\midrule
10 & 20 & 10 & 0.3 & 0.73 $\pm$ 0.03 & 0.57 $\pm$ 0.04 & 1.00 $\pm$ 0.00 & 0.99 $\pm$ 0.00 & 0.73 $\pm$ 0.03 & 0.48 $\pm$ 0.04 & 1.00 $\pm$ 0.00 & 0.98 $\pm$ 0.01 \\
10 & 20 & 10 & 0.7 & 0.99 $\pm$ 0.01 & 0.93 $\pm$ 0.02 & 1.00 $\pm$ 0.00 & 1.00 $\pm$ 0.00 & 0.99 $\pm$ 0.01 & 0.86 $\pm$ 0.04 & 1.00 $\pm$ 0.00 & 0.99 $\pm$ 0.01 \\
10 & 20 & 10 & 1.2 & 1.00 $\pm$ 0.00 & 0.99 $\pm$ 0.01 & 1.00 $\pm$ 0.00 & 1.00 $\pm$ 0.00 & 1.00 $\pm$ 0.00 & 0.97 $\pm$ 0.02 & 1.00 $\pm$ 0.00 & 0.99 $\pm$ 0.01 \\
10 & 20 & 200 & 0.3 & 0.47 $\pm$ 0.04 & 0.29 $\pm$ 0.03 & 0.98 $\pm$ 0.00 & 0.92 $\pm$ 0.01 & 0.47 $\pm$ 0.04 & 0.19 $\pm$ 0.03 & 0.98 $\pm$ 0.00 & 0.83 $\pm$ 0.02 \\
10 & 20 & 200 & 0.7 & 0.83 $\pm$ 0.01 & 0.62 $\pm$ 0.02 & 0.98 $\pm$ 0.00 & 0.91 $\pm$ 0.01 & 0.83 $\pm$ 0.01 & 0.46 $\pm$ 0.02 & 0.98 $\pm$ 0.00 & 0.80 $\pm$ 0.02 \\
10 & 20 & 200 & 1.2 & 0.95 $\pm$ 0.01 & 0.79 $\pm$ 0.01 & 0.99 $\pm$ 0.00 & 0.91 $\pm$ 0.01 & 0.95 $\pm$ 0.01 & 0.62 $\pm$ 0.02 & 0.99 $\pm$ 0.00 & 0.80 $\pm$ 0.02 \\
10 & 60 & 10 & 0.3 & 0.80 $\pm$ 0.03 & 0.56 $\pm$ 0.04 & 1.00 $\pm$ 0.00 & 0.99 $\pm$ 0.00 & 0.80 $\pm$ 0.03 & 0.47 $\pm$ 0.04 & 1.00 $\pm$ 0.00 & 0.98 $\pm$ 0.01 \\
10 & 60 & 10 & 0.7 & 1.00 $\pm$ 0.00 & 0.93 $\pm$ 0.03 & 1.00 $\pm$ 0.00 & 1.00 $\pm$ 0.00 & 1.00 $\pm$ 0.00 & 0.86 $\pm$ 0.04 & 1.00 $\pm$ 0.00 & 0.99 $\pm$ 0.01 \\
10 & 60 & 10 & 1.2 & 1.00 $\pm$ 0.00 & 0.99 $\pm$ 0.01 & 1.00 $\pm$ 0.00 & 1.00 $\pm$ 0.00 & 1.00 $\pm$ 0.00 & 0.97 $\pm$ 0.02 & 1.00 $\pm$ 0.00 & 0.99 $\pm$ 0.00 \\
10 & 60 & 200 & 0.3 & 0.57 $\pm$ 0.04 & 0.27 $\pm$ 0.03 & 0.99 $\pm$ 0.00 & 0.92 $\pm$ 0.01 & 0.57 $\pm$ 0.04 & 0.18 $\pm$ 0.03 & 0.99 $\pm$ 0.00 & 0.83 $\pm$ 0.02 \\
10 & 60 & 200 & 0.7 & 0.93 $\pm$ 0.01 & 0.60 $\pm$ 0.02 & 1.00 $\pm$ 0.00 & 0.91 $\pm$ 0.01 & 0.93 $\pm$ 0.01 & 0.44 $\pm$ 0.02 & 1.00 $\pm$ 0.00 & 0.80 $\pm$ 0.02 \\
10 & 60 & 200 & 1.2 & 0.99 $\pm$ 0.00 & 0.78 $\pm$ 0.01 & 1.00 $\pm$ 0.00 & 0.92 $\pm$ 0.01 & 0.99 $\pm$ 0.00 & 0.60 $\pm$ 0.02 & 1.00 $\pm$ 0.00 & 0.79 $\pm$ 0.02 \\
30 & 20 & 10 & 0.3 & 0.49 $\pm$ 0.02 & 0.39 $\pm$ 0.02 & 1.00 $\pm$ 0.00 & 1.00 $\pm$ 0.00 & 0.49 $\pm$ 0.02 & 0.33 $\pm$ 0.02 & 1.00 $\pm$ 0.00 & 0.99 $\pm$ 0.01 \\
30 & 20 & 10 & 0.7 & 0.85 $\pm$ 0.02 & 0.74 $\pm$ 0.03 & 1.00 $\pm$ 0.00 & 1.00 $\pm$ 0.00 & 0.85 $\pm$ 0.02 & 0.66 $\pm$ 0.04 & 1.00 $\pm$ 0.00 & 0.99 $\pm$ 0.01 \\
30 & 20 & 10 & 1.2 & 0.98 $\pm$ 0.01 & 0.92 $\pm$ 0.02 & 1.00 $\pm$ 0.00 & 1.00 $\pm$ 0.00 & 0.98 $\pm$ 0.01 & 0.87 $\pm$ 0.03 & 1.00 $\pm$ 0.00 & 0.99 $\pm$ 0.00 \\
30 & 20 & 200 & 0.3 & 0.29 $\pm$ 0.01 & 0.19 $\pm$ 0.01 & 0.97 $\pm$ 0.00 & 0.89 $\pm$ 0.01 & 0.29 $\pm$ 0.01 & 0.13 $\pm$ 0.01 & 0.97 $\pm$ 0.00 & 0.79 $\pm$ 0.02 \\
30 & 20 & 200 & 0.7 & 0.59 $\pm$ 0.01 & 0.41 $\pm$ 0.01 & 0.97 $\pm$ 0.00 & 0.87 $\pm$ 0.01 & 0.59 $\pm$ 0.01 & 0.30 $\pm$ 0.01 & 0.97 $\pm$ 0.00 & 0.75 $\pm$ 0.01 \\
30 & 20 & 200 & 1.2 & 0.79 $\pm$ 0.01 & 0.59 $\pm$ 0.01 & 0.97 $\pm$ 0.00 & 0.87 $\pm$ 0.01 & 0.79 $\pm$ 0.01 & 0.45 $\pm$ 0.01 & 0.97 $\pm$ 0.00 & 0.74 $\pm$ 0.01 \\
30 & 60 & 10 & 0.3 & 0.55 $\pm$ 0.02 & 0.38 $\pm$ 0.02 & 1.00 $\pm$ 0.00 & 1.00 $\pm$ 0.00 & 0.55 $\pm$ 0.02 & 0.33 $\pm$ 0.03 & 1.00 $\pm$ 0.00 & 0.99 $\pm$ 0.01 \\
30 & 60 & 10 & 0.7 & 0.91 $\pm$ 0.02 & 0.74 $\pm$ 0.03 & 1.00 $\pm$ 0.00 & 1.00 $\pm$ 0.00 & 0.91 $\pm$ 0.02 & 0.66 $\pm$ 0.04 & 1.00 $\pm$ 0.00 & 0.99 $\pm$ 0.00 \\
30 & 60 & 10 & 1.2 & 1.00 $\pm$ 0.00 & 0.92 $\pm$ 0.02 & 1.00 $\pm$ 0.00 & 1.00 $\pm$ 0.00 & 1.00 $\pm$ 0.00 & 0.87 $\pm$ 0.03 & 1.00 $\pm$ 0.00 & 0.99 $\pm$ 0.00 \\
30 & 60 & 200 & 0.3 & 0.36 $\pm$ 0.01 & 0.18 $\pm$ 0.01 & 0.99 $\pm$ 0.00 & 0.89 $\pm$ 0.01 & 0.36 $\pm$ 0.01 & 0.13 $\pm$ 0.01 & 0.99 $\pm$ 0.00 & 0.79 $\pm$ 0.01 \\
30 & 60 & 200 & 0.7 & 0.69 $\pm$ 0.01 & 0.40 $\pm$ 0.01 & 0.99 $\pm$ 0.00 & 0.88 $\pm$ 0.01 & 0.69 $\pm$ 0.01 & 0.29 $\pm$ 0.01 & 0.99 $\pm$ 0.00 & 0.75 $\pm$ 0.01 \\
30 & 60 & 200 & 1.2 & 0.89 $\pm$ 0.01 & 0.58 $\pm$ 0.01 & 0.99 $\pm$ 0.00 & 0.87 $\pm$ 0.01 & 0.89 $\pm$ 0.01 & 0.44 $\pm$ 0.01 & 0.99 $\pm$ 0.00 & 0.74 $\pm$ 0.01 \\
\bottomrule
\end{tabular}
}

\expandafter\providecommand\csname TiesCorr\endcsname{%
\begin{tabular}{@{}cccccccccccc@{}}
\toprule
\multirow{3}{*}{\# Models} & \multirow{3}{*}{\# Tasks} & \multirow{3}{*}{\# Base values} & \multirow{3}{*}{$\sigma$} & \multicolumn{4}{c}{$\alpha_{ldb}=0.3$} & \multicolumn{4}{c}{$\alpha_{ldb}=0.5$} \\
\cmidrule(lr){5-8}\cmidrule(lr){9-12}
 &  &  &  & \multicolumn{2}{c}{Width mean $\pm$ SD} & \multicolumn{2}{c}{Coverage mean $\pm$ SD} & \multicolumn{2}{c}{Width mean $\pm$ SD} & \multicolumn{2}{c}{Coverage mean $\pm$ SD} \\
\cmidrule(lr){5-6}\cmidrule(lr){7-8}\cmidrule(lr){9-10}\cmidrule(lr){11-12}
 &  &  &  & Union & Quantile & Union & Quantile & Union & Quantile & Union & Quantile \\
\midrule
10 & 20 & 10 & 0.3 & 0.82 $\pm$ 0.10 & 0.70 $\pm$ 0.13 & 0.99 $\pm$ 0.01 & 0.96 $\pm$ 0.02 & 0.82 $\pm$ 0.10 & 0.62 $\pm$ 0.15 & 0.99 $\pm$ 0.01 & 0.92 $\pm$ 0.05 \\
10 & 20 & 10 & 0.7 & 0.99 $\pm$ 0.01 & 0.95 $\pm$ 0.03 & 1.00 $\pm$ 0.00 & 0.99 $\pm$ 0.01 & 0.99 $\pm$ 0.01 & 0.90 $\pm$ 0.05 & 1.00 $\pm$ 0.00 & 0.97 $\pm$ 0.02 \\
10 & 20 & 10 & 1.2 & 1.00 $\pm$ 0.00 & 0.99 $\pm$ 0.01 & 1.00 $\pm$ 0.00 & 1.00 $\pm$ 0.00 & 1.00 $\pm$ 0.00 & 0.98 $\pm$ 0.02 & 1.00 $\pm$ 0.00 & 0.99 $\pm$ 0.01 \\
10 & 20 & 200 & 0.3 & 0.71 $\pm$ 0.20 & 0.57 $\pm$ 0.24 & 0.98 $\pm$ 0.01 & 0.91 $\pm$ 0.02 & 0.71 $\pm$ 0.20 & 0.48 $\pm$ 0.25 & 0.98 $\pm$ 0.01 & 0.82 $\pm$ 0.04 \\
10 & 20 & 200 & 0.7 & 0.91 $\pm$ 0.07 & 0.77 $\pm$ 0.14 & 0.98 $\pm$ 0.01 & 0.91 $\pm$ 0.03 & 0.91 $\pm$ 0.07 & 0.66 $\pm$ 0.19 & 0.98 $\pm$ 0.01 & 0.80 $\pm$ 0.05 \\
10 & 20 & 200 & 1.2 & 0.97 $\pm$ 0.02 & 0.88 $\pm$ 0.08 & 0.99 $\pm$ 0.01 & 0.93 $\pm$ 0.03 & 0.97 $\pm$ 0.02 & 0.76 $\pm$ 0.14 & 0.99 $\pm$ 0.01 & 0.81 $\pm$ 0.05 \\
10 & 60 & 10 & 0.3 & 0.88 $\pm$ 0.08 & 0.69 $\pm$ 0.14 & 1.00 $\pm$ 0.00 & 0.96 $\pm$ 0.03 & 0.88 $\pm$ 0.08 & 0.61 $\pm$ 0.15 & 1.00 $\pm$ 0.00 & 0.92 $\pm$ 0.05 \\
10 & 60 & 10 & 0.7 & 1.00 $\pm$ 0.00 & 0.96 $\pm$ 0.03 & 1.00 $\pm$ 0.00 & 0.99 $\pm$ 0.00 & 1.00 $\pm$ 0.00 & 0.90 $\pm$ 0.06 & 1.00 $\pm$ 0.00 & 0.98 $\pm$ 0.02 \\
10 & 60 & 10 & 1.2 & 1.00 $\pm$ 0.00 & 1.00 $\pm$ 0.00 & 1.00 $\pm$ 0.00 & 1.00 $\pm$ 0.00 & 1.00 $\pm$ 0.00 & 0.98 $\pm$ 0.02 & 1.00 $\pm$ 0.00 & 0.99 $\pm$ 0.01 \\
10 & 60 & 200 & 0.3 & 0.77 $\pm$ 0.17 & 0.56 $\pm$ 0.24 & 0.99 $\pm$ 0.00 & 0.91 $\pm$ 0.02 & 0.77 $\pm$ 0.17 & 0.47 $\pm$ 0.25 & 0.99 $\pm$ 0.00 & 0.82 $\pm$ 0.04 \\
10 & 60 & 200 & 0.7 & 0.96 $\pm$ 0.03 & 0.77 $\pm$ 0.15 & 1.00 $\pm$ 0.00 & 0.92 $\pm$ 0.03 & 0.96 $\pm$ 0.03 & 0.65 $\pm$ 0.20 & 1.00 $\pm$ 0.00 & 0.81 $\pm$ 0.05 \\
10 & 60 & 200 & 1.2 & 1.00 $\pm$ 0.00 & 0.87 $\pm$ 0.09 & 1.00 $\pm$ 0.00 & 0.93 $\pm$ 0.04 & 1.00 $\pm$ 0.00 & 0.75 $\pm$ 0.14 & 1.00 $\pm$ 0.00 & 0.81 $\pm$ 0.06 \\
30 & 20 & 10 & 0.3 & 0.67 $\pm$ 0.16 & 0.57 $\pm$ 0.17 & 0.98 $\pm$ 0.01 & 0.93 $\pm$ 0.04 & 0.67 $\pm$ 0.16 & 0.52 $\pm$ 0.18 & 0.98 $\pm$ 0.01 & 0.88 $\pm$ 0.08 \\
30 & 20 & 10 & 0.7 & 0.90 $\pm$ 0.05 & 0.81 $\pm$ 0.08 & 0.99 $\pm$ 0.01 & 0.96 $\pm$ 0.04 & 0.90 $\pm$ 0.05 & 0.74 $\pm$ 0.09 & 0.99 $\pm$ 0.01 & 0.90 $\pm$ 0.08 \\
30 & 20 & 10 & 1.2 & 0.99 $\pm$ 0.01 & 0.95 $\pm$ 0.03 & 1.00 $\pm$ 0.00 & 0.99 $\pm$ 0.01 & 0.99 $\pm$ 0.01 & 0.90 $\pm$ 0.04 & 1.00 $\pm$ 0.00 & 0.96 $\pm$ 0.04 \\
30 & 20 & 200 & 0.3 & 0.59 $\pm$ 0.24 & 0.49 $\pm$ 0.25 & 0.96 $\pm$ 0.01 & 0.87 $\pm$ 0.02 & 0.59 $\pm$ 0.24 & 0.43 $\pm$ 0.25 & 0.96 $\pm$ 0.01 & 0.77 $\pm$ 0.04 \\
30 & 20 & 200 & 0.7 & 0.78 $\pm$ 0.16 & 0.64 $\pm$ 0.20 & 0.96 $\pm$ 0.01 & 0.85 $\pm$ 0.02 & 0.78 $\pm$ 0.16 & 0.55 $\pm$ 0.22 & 0.96 $\pm$ 0.01 & 0.71 $\pm$ 0.03 \\
30 & 20 & 200 & 1.2 & 0.88 $\pm$ 0.08 & 0.75 $\pm$ 0.14 & 0.96 $\pm$ 0.01 & 0.85 $\pm$ 0.02 & 0.88 $\pm$ 0.08 & 0.65 $\pm$ 0.18 & 0.96 $\pm$ 0.01 & 0.70 $\pm$ 0.03 \\
30 & 60 & 10 & 0.3 & 0.72 $\pm$ 0.15 & 0.57 $\pm$ 0.17 & 0.99 $\pm$ 0.00 & 0.93 $\pm$ 0.04 & 0.72 $\pm$ 0.15 & 0.51 $\pm$ 0.18 & 0.99 $\pm$ 0.00 & 0.88 $\pm$ 0.08 \\
30 & 60 & 10 & 0.7 & 0.95 $\pm$ 0.04 & 0.81 $\pm$ 0.08 & 1.00 $\pm$ 0.00 & 0.96 $\pm$ 0.03 & 0.95 $\pm$ 0.04 & 0.74 $\pm$ 0.09 & 1.00 $\pm$ 0.00 & 0.91 $\pm$ 0.08 \\
30 & 60 & 10 & 1.2 & 1.00 $\pm$ 0.00 & 0.95 $\pm$ 0.03 & 1.00 $\pm$ 0.00 & 0.99 $\pm$ 0.01 & 1.00 $\pm$ 0.00 & 0.90 $\pm$ 0.04 & 1.00 $\pm$ 0.00 & 0.96 $\pm$ 0.04 \\
30 & 60 & 200 & 0.3 & 0.65 $\pm$ 0.23 & 0.48 $\pm$ 0.25 & 0.99 $\pm$ 0.00 & 0.87 $\pm$ 0.02 & 0.65 $\pm$ 0.23 & 0.42 $\pm$ 0.26 & 0.99 $\pm$ 0.00 & 0.77 $\pm$ 0.04 \\
30 & 60 & 200 & 0.7 & 0.85 $\pm$ 0.12 & 0.64 $\pm$ 0.20 & 0.99 $\pm$ 0.00 & 0.86 $\pm$ 0.03 & 0.85 $\pm$ 0.12 & 0.55 $\pm$ 0.22 & 0.99 $\pm$ 0.00 & 0.71 $\pm$ 0.04 \\
30 & 60 & 200 & 1.2 & 0.94 $\pm$ 0.04 & 0.74 $\pm$ 0.14 & 0.99 $\pm$ 0.00 & 0.85 $\pm$ 0.02 & 0.94 $\pm$ 0.04 & 0.64 $\pm$ 0.18 & 0.99 $\pm$ 0.00 & 0.70 $\pm$ 0.03 \\
\bottomrule
\end{tabular}
}

\begin{table}[ht]
    \centering
        \caption{Comparison of normalized width and coverage of leaderboard-level PIs for $\alpha_{ldb}=0.3$ and $\alpha_{ldb}=0.5$, without correlations and without ties between models.}
            \label{tab:no_ties_no_corr}
        \resizebox{\textwidth}{!}{%
            \csname NoTiesNoCorr\endcsname
        }
\end{table}

\begin{table}[ht]
    \centering
        \caption{Comparison of normalized width and coverage of leaderboard-level PIs for $\alpha_{ldb}=0.3$ and $\alpha_{ldb}=0.5$, without correlations and with ties between models.}
            \label{tab:ties_no_corr}
        \resizebox{\textwidth}{!}{%
            \csname TiesNoCorr\endcsname
        }
\end{table}

\begin{table}[ht]
    \centering
        \caption{Comparison of normalized width and coverage of leaderboard-level PIs for $\alpha_{ldb}=0.3$ and $\alpha_{ldb}=0.5$ , with correlations and without ties between models.}
            \label{tab:no_ties_corr}
        \resizebox{\textwidth}{!}{%
            \csname NoTiesCorr\endcsname
        }
\end{table}

\begin{table}[ht]
    \centering
        \caption{Comparison of width and coverage of leaderboard-level PIs for $\alpha_{ldb}=0.3$ and $\alpha_{ldb}=0.5$, with correlations and with ties between models.}
            \label{tab:ties_corr}
        \resizebox{\textwidth}{!}{%
            \csname TiesCorr\endcsname
        }
\end{table}

\subsection{Comparison of FWER and bootstrap task-level rank CIs}
In Tables~\ref{tab:task_no_corr} and~\ref{tab:task_corr}, we present the results of the single task simulations, as described in Section~\ref{app:synthetic_desc}. In general, the bootstrap rank CIs are narrower than the FWER rank CIs, but do not maintain coverage guarantee for most simulated configurations.

\expandafter\providecommand\csname TaskNoCorr\endcsname{%
\begin{tabular}{@{}cccccccccccc@{}}
\toprule
\multirow{2}{*}{$1-\alpha_{tsk}$} & \multirow{2}{*}{\# Models} & \multirow{2}{*}{\# Base values} & \multirow{2}{*}{$\sigma$} & \multicolumn{4}{c}{No ties} & \multicolumn{4}{c}{With ties} \\
\cmidrule(lr){5-8}\cmidrule(lr){9-12}
 &  &  &  & Width FWER & Width Bootstrap & Coverage FWER & Coverage Bootstrap & Width FWER & Width Bootstrap & Coverage FWER & Coverage Bootstrap \\
\midrule
0.98 & 10 & 10 & 0.3 & 0.30 $\pm$ 0.00 & 0.16 $\pm$ 0.00 & 1.00 $\pm$ 0.00 & 1.00 $\pm$ 0.00 & 0.42 $\pm$ 0.19 & 0.32 $\pm$ 0.21 & 0.99 $\pm$ 0.01 & 0.78 $\pm$ 0.23 \\
0.98 & 10 & 10 & 0.7 & 0.51 $\pm$ 0.00 & 0.27 $\pm$ 0.00 & 1.00 $\pm$ 0.00 & 0.99 $\pm$ 0.00 & 0.56 $\pm$ 0.08 & 0.38 $\pm$ 0.15 & 0.99 $\pm$ 0.01 & 0.78 $\pm$ 0.23 \\
0.98 & 10 & 10 & 1.2 & 0.63 $\pm$ 0.00 & 0.32 $\pm$ 0.00 & 1.00 $\pm$ 0.00 & 0.99 $\pm$ 0.00 & 0.68 $\pm$ 0.06 & 0.41 $\pm$ 0.13 & 1.00 $\pm$ 0.01 & 0.79 $\pm$ 0.22 \\
0.98 & 10 & 200 & 0.3 & 0.02 $\pm$ 0.00 & 0.02 $\pm$ 0.00 & 1.00 $\pm$ 0.00 & 1.00 $\pm$ 0.00 & 0.29 $\pm$ 0.30 & 0.27 $\pm$ 0.28 & 0.99 $\pm$ 0.01 & 0.86 $\pm$ 0.15 \\
0.98 & 10 & 200 & 0.7 & 0.09 $\pm$ 0.00 & 0.07 $\pm$ 0.00 & 1.00 $\pm$ 0.00 & 1.00 $\pm$ 0.00 & 0.32 $\pm$ 0.28 & 0.29 $\pm$ 0.26 & 0.99 $\pm$ 0.01 & 0.86 $\pm$ 0.16 \\
0.98 & 10 & 200 & 1.2 & 0.08 $\pm$ 0.00 & 0.07 $\pm$ 0.00 & 1.00 $\pm$ 0.00 & 1.00 $\pm$ 0.00 & 0.31 $\pm$ 0.28 & 0.29 $\pm$ 0.26 & 0.99 $\pm$ 0.00 & 0.87 $\pm$ 0.15 \\
0.98 & 30 & 10 & 0.3 & 0.21 $\pm$ 0.00 & 0.10 $\pm$ 0.00 & 1.00 $\pm$ 0.00 & 0.99 $\pm$ 0.00 & 0.39 $\pm$ 0.22 & 0.28 $\pm$ 0.21 & 0.99 $\pm$ 0.01 & 0.61 $\pm$ 0.36 \\
0.98 & 30 & 10 & 0.7 & 0.43 $\pm$ 0.00 & 0.18 $\pm$ 0.00 & 1.00 $\pm$ 0.00 & 0.98 $\pm$ 0.00 & 0.51 $\pm$ 0.12 & 0.31 $\pm$ 0.18 & 0.99 $\pm$ 0.01 & 0.61 $\pm$ 0.36 \\
0.98 & 30 & 10 & 1.2 & 0.57 $\pm$ 0.00 & 0.24 $\pm$ 0.00 & 1.00 $\pm$ 0.00 & 0.98 $\pm$ 0.00 & 0.62 $\pm$ 0.07 & 0.35 $\pm$ 0.15 & 1.00 $\pm$ 0.01 & 0.62 $\pm$ 0.35 \\
0.98 & 30 & 200 & 0.3 & 0.04 $\pm$ 0.00 & 0.03 $\pm$ 0.00 & 1.00 $\pm$ 0.00 & 1.00 $\pm$ 0.00 & 0.31 $\pm$ 0.30 & 0.27 $\pm$ 0.27 & 0.99 $\pm$ 0.00 & 0.65 $\pm$ 0.35 \\
0.98 & 30 & 200 & 0.7 & 0.07 $\pm$ 0.00 & 0.05 $\pm$ 0.00 & 1.00 $\pm$ 0.00 & 1.00 $\pm$ 0.00 & 0.32 $\pm$ 0.29 & 0.28 $\pm$ 0.26 & 0.99 $\pm$ 0.01 & 0.65 $\pm$ 0.35 \\
0.98 & 30 & 200 & 1.2 & 0.09 $\pm$ 0.00 & 0.07 $\pm$ 0.00 & 1.00 $\pm$ 0.00 & 1.00 $\pm$ 0.00 & 0.33 $\pm$ 0.28 & 0.28 $\pm$ 0.25 & 0.99 $\pm$ 0.01 & 0.64 $\pm$ 0.35 \\
\midrule
0.95 & 10 & 10 & 0.3 & 0.25 $\pm$ 0.00 & 0.14 $\pm$ 0.00 & 1.00 $\pm$ 0.00 & 0.99 $\pm$ 0.00 & 0.39 $\pm$ 0.22 & 0.29 $\pm$ 0.20 & 0.98 $\pm$ 0.01 & 0.71 $\pm$ 0.29 \\
0.95 & 10 & 10 & 0.7 & 0.42 $\pm$ 0.00 & 0.23 $\pm$ 0.00 & 1.00 $\pm$ 0.00 & 0.98 $\pm$ 0.00 & 0.50 $\pm$ 0.12 & 0.34 $\pm$ 0.15 & 0.98 $\pm$ 0.02 & 0.71 $\pm$ 0.29 \\
0.95 & 10 & 10 & 1.2 & 0.54 $\pm$ 0.00 & 0.27 $\pm$ 0.00 & 1.00 $\pm$ 0.00 & 0.98 $\pm$ 0.00 & 0.60 $\pm$ 0.09 & 0.36 $\pm$ 0.13 & 0.99 $\pm$ 0.01 & 0.71 $\pm$ 0.28 \\
0.95 & 10 & 200 & 0.3 & 0.02 $\pm$ 0.00 & 0.01 $\pm$ 0.00 & 1.00 $\pm$ 0.00 & 1.00 $\pm$ 0.00 & 0.29 $\pm$ 0.30 & 0.25 $\pm$ 0.26 & 0.99 $\pm$ 0.01 & 0.77 $\pm$ 0.25 \\
0.95 & 10 & 200 & 0.7 & 0.08 $\pm$ 0.00 & 0.06 $\pm$ 0.00 & 1.00 $\pm$ 0.00 & 1.00 $\pm$ 0.00 & 0.31 $\pm$ 0.28 & 0.26 $\pm$ 0.24 & 0.98 $\pm$ 0.02 & 0.76 $\pm$ 0.26 \\
0.95 & 10 & 200 & 1.2 & 0.07 $\pm$ 0.00 & 0.05 $\pm$ 0.00 & 1.00 $\pm$ 0.00 & 1.00 $\pm$ 0.00 & 0.31 $\pm$ 0.28 & 0.26 $\pm$ 0.24 & 0.98 $\pm$ 0.01 & 0.77 $\pm$ 0.25 \\
0.95 & 30 & 10 & 0.3 & 0.18 $\pm$ 0.00 & 0.08 $\pm$ 0.00 & 1.00 $\pm$ 0.00 & 0.97 $\pm$ 0.00 & 0.37 $\pm$ 0.24 & 0.25 $\pm$ 0.20 & 0.99 $\pm$ 0.01 & 0.57 $\pm$ 0.38 \\
0.95 & 30 & 10 & 0.7 & 0.37 $\pm$ 0.00 & 0.15 $\pm$ 0.00 & 1.00 $\pm$ 0.00 & 0.96 $\pm$ 0.00 & 0.47 $\pm$ 0.15 & 0.28 $\pm$ 0.17 & 0.99 $\pm$ 0.01 & 0.57 $\pm$ 0.37 \\
0.95 & 30 & 10 & 1.2 & 0.50 $\pm$ 0.00 & 0.20 $\pm$ 0.00 & 1.00 $\pm$ 0.00 & 0.96 $\pm$ 0.00 & 0.57 $\pm$ 0.09 & 0.31 $\pm$ 0.14 & 0.99 $\pm$ 0.01 & 0.57 $\pm$ 0.37 \\
0.95 & 30 & 200 & 0.3 & 0.03 $\pm$ 0.00 & 0.02 $\pm$ 0.00 & 1.00 $\pm$ 0.00 & 1.00 $\pm$ 0.00 & 0.31 $\pm$ 0.30 & 0.24 $\pm$ 0.24 & 0.99 $\pm$ 0.01 & 0.58 $\pm$ 0.39 \\
0.95 & 30 & 200 & 0.7 & 0.06 $\pm$ 0.00 & 0.04 $\pm$ 0.00 & 1.00 $\pm$ 0.00 & 1.00 $\pm$ 0.00 & 0.32 $\pm$ 0.29 & 0.25 $\pm$ 0.23 & 0.99 $\pm$ 0.01 & 0.57 $\pm$ 0.39 \\
0.95 & 30 & 200 & 1.2 & 0.09 $\pm$ 0.00 & 0.05 $\pm$ 0.00 & 1.00 $\pm$ 0.00 & 0.99 $\pm$ 0.00 & 0.33 $\pm$ 0.28 & 0.25 $\pm$ 0.23 & 0.99 $\pm$ 0.01 & 0.57 $\pm$ 0.39 \\
\midrule
0.90 & 10 & 10 & 0.3 & 0.21 $\pm$ 0.00 & 0.11 $\pm$ 0.00 & 1.00 $\pm$ 0.00 & 0.98 $\pm$ 0.00 & 0.37 $\pm$ 0.23 & 0.26 $\pm$ 0.18 & 0.97 $\pm$ 0.03 & 0.64 $\pm$ 0.34 \\
0.90 & 10 & 10 & 0.7 & 0.37 $\pm$ 0.00 & 0.19 $\pm$ 0.00 & 1.00 $\pm$ 0.00 & 0.96 $\pm$ 0.00 & 0.46 $\pm$ 0.15 & 0.30 $\pm$ 0.14 & 0.97 $\pm$ 0.03 & 0.63 $\pm$ 0.33 \\
0.90 & 10 & 10 & 1.2 & 0.47 $\pm$ 0.00 & 0.22 $\pm$ 0.00 & 1.00 $\pm$ 0.00 & 0.97 $\pm$ 0.00 & 0.54 $\pm$ 0.11 & 0.31 $\pm$ 0.13 & 0.97 $\pm$ 0.03 & 0.64 $\pm$ 0.33 \\
0.90 & 10 & 200 & 0.3 & 0.02 $\pm$ 0.00 & 0.01 $\pm$ 0.00 & 1.00 $\pm$ 0.00 & 1.00 $\pm$ 0.00 & 0.29 $\pm$ 0.30 & 0.22 $\pm$ 0.23 & 0.97 $\pm$ 0.02 & 0.67 $\pm$ 0.34 \\
0.90 & 10 & 200 & 0.7 & 0.07 $\pm$ 0.00 & 0.05 $\pm$ 0.00 & 1.00 $\pm$ 0.00 & 1.00 $\pm$ 0.00 & 0.30 $\pm$ 0.28 & 0.23 $\pm$ 0.21 & 0.97 $\pm$ 0.03 & 0.67 $\pm$ 0.34 \\
0.90 & 10 & 200 & 1.2 & 0.06 $\pm$ 0.00 & 0.05 $\pm$ 0.00 & 1.00 $\pm$ 0.00 & 1.00 $\pm$ 0.00 & 0.30 $\pm$ 0.28 & 0.24 $\pm$ 0.22 & 0.97 $\pm$ 0.02 & 0.67 $\pm$ 0.34 \\
0.90 & 30 & 10 & 0.3 & 0.16 $\pm$ 0.00 & 0.07 $\pm$ 0.00 & 1.00 $\pm$ 0.00 & 0.95 $\pm$ 0.00 & 0.36 $\pm$ 0.25 & 0.22 $\pm$ 0.18 & 0.97 $\pm$ 0.03 & 0.55 $\pm$ 0.37 \\
0.90 & 30 & 10 & 0.7 & 0.33 $\pm$ 0.00 & 0.12 $\pm$ 0.00 & 1.00 $\pm$ 0.00 & 0.94 $\pm$ 0.00 & 0.44 $\pm$ 0.17 & 0.25 $\pm$ 0.15 & 0.97 $\pm$ 0.03 & 0.54 $\pm$ 0.37 \\
0.90 & 30 & 10 & 1.2 & 0.45 $\pm$ 0.00 & 0.17 $\pm$ 0.00 & 1.00 $\pm$ 0.00 & 0.93 $\pm$ 0.00 & 0.53 $\pm$ 0.11 & 0.27 $\pm$ 0.13 & 0.98 $\pm$ 0.03 & 0.53 $\pm$ 0.36 \\
0.90 & 30 & 200 & 0.3 & 0.03 $\pm$ 0.00 & 0.02 $\pm$ 0.00 & 1.00 $\pm$ 0.00 & 1.00 $\pm$ 0.00 & 0.30 $\pm$ 0.30 & 0.21 $\pm$ 0.21 & 0.97 $\pm$ 0.02 & 0.56 $\pm$ 0.39 \\
0.90 & 30 & 200 & 0.7 & 0.06 $\pm$ 0.00 & 0.03 $\pm$ 0.00 & 1.00 $\pm$ 0.00 & 0.99 $\pm$ 0.00 & 0.32 $\pm$ 0.29 & 0.22 $\pm$ 0.21 & 0.98 $\pm$ 0.02 & 0.55 $\pm$ 0.39 \\
0.90 & 30 & 200 & 1.2 & 0.08 $\pm$ 0.00 & 0.05 $\pm$ 0.00 & 1.00 $\pm$ 0.00 & 0.98 $\pm$ 0.00 & 0.32 $\pm$ 0.28 & 0.22 $\pm$ 0.20 & 0.97 $\pm$ 0.02 & 0.55 $\pm$ 0.39 \\
\bottomrule
\end{tabular}
}

\expandafter\providecommand\csname TaskCorr\endcsname{%
\begin{tabular}{@{}cccccccccccc@{}}
\toprule
\multirow{2}{*}{$1-\alpha_{tsk}$} & \multirow{2}{*}{\# Models} & \multirow{2}{*}{\# Base values} & \multirow{2}{*}{$\sigma$} & \multicolumn{4}{c}{No ties} & \multicolumn{4}{c}{With ties} \\
\cmidrule(lr){5-8}\cmidrule(lr){9-12}
 &  &  &  & Width FWER & Width Bootstrap & Coverage FWER & Coverage Bootstrap & Width FWER & Width Bootstrap & Coverage FWER & Coverage Bootstrap \\
\midrule
0.98 & 10 & 10 & 0.3 & 0.26 $\pm$ 0.04 & 0.14 $\pm$ 0.02 & 1.00 $\pm$ 0.00 & 0.99 $\pm$ 0.00 & 0.40 $\pm$ 0.17 & 0.31 $\pm$ 0.18 & 0.99 $\pm$ 0.01 & 0.79 $\pm$ 0.18 \\
0.98 & 10 & 10 & 0.7 & 0.48 $\pm$ 0.04 & 0.25 $\pm$ 0.02 & 1.00 $\pm$ 0.00 & 0.98 $\pm$ 0.01 & 0.56 $\pm$ 0.08 & 0.37 $\pm$ 0.13 & 0.99 $\pm$ 0.01 & 0.79 $\pm$ 0.18 \\
0.98 & 10 & 10 & 1.2 & 0.62 $\pm$ 0.04 & 0.32 $\pm$ 0.02 & 1.00 $\pm$ 0.00 & 0.99 $\pm$ 0.00 & 0.67 $\pm$ 0.07 & 0.41 $\pm$ 0.11 & 1.00 $\pm$ 0.00 & 0.79 $\pm$ 0.18 \\
0.98 & 10 & 200 & 0.3 & 0.04 $\pm$ 0.02 & 0.03 $\pm$ 0.02 & 1.00 $\pm$ 0.00 & 1.00 $\pm$ 0.00 & 0.30 $\pm$ 0.25 & 0.27 $\pm$ 0.23 & 0.99 $\pm$ 0.00 & 0.87 $\pm$ 0.13 \\
0.98 & 10 & 200 & 0.7 & 0.10 $\pm$ 0.03 & 0.08 $\pm$ 0.02 & 1.00 $\pm$ 0.00 & 1.00 $\pm$ 0.00 & 0.32 $\pm$ 0.23 & 0.29 $\pm$ 0.21 & 0.99 $\pm$ 0.00 & 0.87 $\pm$ 0.13 \\
0.98 & 10 & 200 & 1.2 & 0.09 $\pm$ 0.02 & 0.07 $\pm$ 0.02 & 1.00 $\pm$ 0.00 & 1.00 $\pm$ 0.00 & 0.32 $\pm$ 0.23 & 0.29 $\pm$ 0.21 & 0.99 $\pm$ 0.00 & 0.87 $\pm$ 0.12 \\
0.98 & 30 & 10 & 0.3 & 0.21 $\pm$ 0.01 & 0.09 $\pm$ 0.01 & 1.00 $\pm$ 0.00 & 0.99 $\pm$ 0.00 & 0.38 $\pm$ 0.19 & 0.28 $\pm$ 0.18 & 0.99 $\pm$ 0.00 & 0.62 $\pm$ 0.30 \\
0.98 & 30 & 10 & 0.7 & 0.42 $\pm$ 0.03 & 0.17 $\pm$ 0.01 & 1.00 $\pm$ 0.00 & 0.98 $\pm$ 0.00 & 0.51 $\pm$ 0.11 & 0.31 $\pm$ 0.15 & 0.99 $\pm$ 0.00 & 0.61 $\pm$ 0.29 \\
0.98 & 30 & 10 & 1.2 & 0.57 $\pm$ 0.05 & 0.24 $\pm$ 0.02 & 1.00 $\pm$ 0.00 & 0.98 $\pm$ 0.00 & 0.62 $\pm$ 0.07 & 0.35 $\pm$ 0.13 & 1.00 $\pm$ 0.00 & 0.62 $\pm$ 0.29 \\
0.98 & 30 & 200 & 0.3 & 0.03 $\pm$ 0.01 & 0.02 $\pm$ 0.01 & 1.00 $\pm$ 0.00 & 1.00 $\pm$ 0.00 & 0.31 $\pm$ 0.25 & 0.27 $\pm$ 0.22 & 0.99 $\pm$ 0.00 & 0.65 $\pm$ 0.29 \\
0.98 & 30 & 200 & 0.7 & 0.07 $\pm$ 0.01 & 0.04 $\pm$ 0.01 & 1.00 $\pm$ 0.00 & 1.00 $\pm$ 0.00 & 0.32 $\pm$ 0.24 & 0.28 $\pm$ 0.21 & 0.99 $\pm$ 0.00 & 0.65 $\pm$ 0.29 \\
0.98 & 30 & 200 & 1.2 & 0.09 $\pm$ 0.02 & 0.06 $\pm$ 0.01 & 1.00 $\pm$ 0.00 & 1.00 $\pm$ 0.00 & 0.33 $\pm$ 0.23 & 0.28 $\pm$ 0.21 & 0.99 $\pm$ 0.00 & 0.65 $\pm$ 0.29 \\
\midrule
0.95 & 10 & 10 & 0.3 & 0.22 $\pm$ 0.03 & 0.12 $\pm$ 0.02 & 1.00 $\pm$ 0.00 & 0.99 $\pm$ 0.01 & 0.38 $\pm$ 0.19 & 0.28 $\pm$ 0.17 & 0.98 $\pm$ 0.01 & 0.71 $\pm$ 0.24 \\
0.95 & 10 & 10 & 0.7 & 0.40 $\pm$ 0.04 & 0.22 $\pm$ 0.02 & 1.00 $\pm$ 0.00 & 0.97 $\pm$ 0.01 & 0.50 $\pm$ 0.11 & 0.33 $\pm$ 0.13 & 0.98 $\pm$ 0.01 & 0.71 $\pm$ 0.23 \\
0.95 & 10 & 10 & 1.2 & 0.52 $\pm$ 0.04 & 0.26 $\pm$ 0.02 & 1.00 $\pm$ 0.00 & 0.98 $\pm$ 0.00 & 0.60 $\pm$ 0.09 & 0.36 $\pm$ 0.11 & 0.99 $\pm$ 0.01 & 0.71 $\pm$ 0.23 \\
0.95 & 10 & 200 & 0.3 & 0.03 $\pm$ 0.02 & 0.02 $\pm$ 0.02 & 1.00 $\pm$ 0.00 & 1.00 $\pm$ 0.00 & 0.29 $\pm$ 0.25 & 0.25 $\pm$ 0.21 & 0.98 $\pm$ 0.01 & 0.77 $\pm$ 0.21 \\
0.95 & 10 & 200 & 0.7 & 0.09 $\pm$ 0.03 & 0.06 $\pm$ 0.02 & 1.00 $\pm$ 0.00 & 1.00 $\pm$ 0.00 & 0.31 $\pm$ 0.23 & 0.27 $\pm$ 0.20 & 0.99 $\pm$ 0.01 & 0.77 $\pm$ 0.21 \\
0.95 & 10 & 200 & 1.2 & 0.08 $\pm$ 0.02 & 0.06 $\pm$ 0.02 & 1.00 $\pm$ 0.00 & 1.00 $\pm$ 0.00 & 0.31 $\pm$ 0.23 & 0.27 $\pm$ 0.20 & 0.99 $\pm$ 0.01 & 0.77 $\pm$ 0.21 \\
0.95 & 30 & 10 & 0.3 & 0.18 $\pm$ 0.01 & 0.08 $\pm$ 0.01 & 1.00 $\pm$ 0.00 & 0.98 $\pm$ 0.01 & 0.37 $\pm$ 0.20 & 0.25 $\pm$ 0.16 & 0.99 $\pm$ 0.01 & 0.57 $\pm$ 0.31 \\
0.95 & 30 & 10 & 0.7 & 0.36 $\pm$ 0.03 & 0.14 $\pm$ 0.01 & 1.00 $\pm$ 0.00 & 0.97 $\pm$ 0.00 & 0.47 $\pm$ 0.13 & 0.28 $\pm$ 0.14 & 0.99 $\pm$ 0.01 & 0.57 $\pm$ 0.31 \\
0.95 & 30 & 10 & 1.2 & 0.50 $\pm$ 0.04 & 0.20 $\pm$ 0.02 & 1.00 $\pm$ 0.00 & 0.96 $\pm$ 0.01 & 0.56 $\pm$ 0.09 & 0.31 $\pm$ 0.12 & 0.99 $\pm$ 0.01 & 0.57 $\pm$ 0.31 \\
0.95 & 30 & 200 & 0.3 & 0.03 $\pm$ 0.01 & 0.02 $\pm$ 0.00 & 1.00 $\pm$ 0.00 & 1.00 $\pm$ 0.00 & 0.30 $\pm$ 0.25 & 0.24 $\pm$ 0.20 & 0.99 $\pm$ 0.01 & 0.57 $\pm$ 0.32 \\
0.95 & 30 & 200 & 0.7 & 0.06 $\pm$ 0.01 & 0.04 $\pm$ 0.01 & 1.00 $\pm$ 0.00 & 1.00 $\pm$ 0.00 & 0.32 $\pm$ 0.24 & 0.25 $\pm$ 0.20 & 0.99 $\pm$ 0.01 & 0.58 $\pm$ 0.33 \\
0.95 & 30 & 200 & 1.2 & 0.08 $\pm$ 0.01 & 0.05 $\pm$ 0.01 & 1.00 $\pm$ 0.00 & 0.99 $\pm$ 0.00 & 0.33 $\pm$ 0.24 & 0.25 $\pm$ 0.19 & 0.99 $\pm$ 0.01 & 0.57 $\pm$ 0.32 \\
\midrule
0.90 & 10 & 10 & 0.3 & 0.19 $\pm$ 0.03 & 0.10 $\pm$ 0.02 & 1.00 $\pm$ 0.00 & 0.98 $\pm$ 0.01 & 0.36 $\pm$ 0.20 & 0.25 $\pm$ 0.15 & 0.97 $\pm$ 0.02 & 0.64 $\pm$ 0.28 \\
0.90 & 10 & 10 & 0.7 & 0.34 $\pm$ 0.03 & 0.18 $\pm$ 0.02 & 0.99 $\pm$ 0.00 & 0.96 $\pm$ 0.02 & 0.46 $\pm$ 0.13 & 0.29 $\pm$ 0.12 & 0.97 $\pm$ 0.02 & 0.64 $\pm$ 0.27 \\
0.90 & 10 & 10 & 1.2 & 0.45 $\pm$ 0.04 & 0.22 $\pm$ 0.02 & 1.00 $\pm$ 0.00 & 0.96 $\pm$ 0.01 & 0.54 $\pm$ 0.10 & 0.31 $\pm$ 0.11 & 0.97 $\pm$ 0.02 & 0.64 $\pm$ 0.27 \\
0.90 & 10 & 200 & 0.3 & 0.03 $\pm$ 0.02 & 0.02 $\pm$ 0.01 & 1.00 $\pm$ 0.00 & 1.00 $\pm$ 0.00 & 0.29 $\pm$ 0.25 & 0.23 $\pm$ 0.19 & 0.97 $\pm$ 0.02 & 0.67 $\pm$ 0.28 \\
0.90 & 10 & 200 & 0.7 & 0.08 $\pm$ 0.02 & 0.05 $\pm$ 0.02 & 1.00 $\pm$ 0.00 & 1.00 $\pm$ 0.00 & 0.31 $\pm$ 0.23 & 0.24 $\pm$ 0.18 & 0.97 $\pm$ 0.02 & 0.67 $\pm$ 0.28 \\
0.90 & 10 & 200 & 1.2 & 0.07 $\pm$ 0.02 & 0.05 $\pm$ 0.02 & 1.00 $\pm$ 0.00 & 1.00 $\pm$ 0.00 & 0.31 $\pm$ 0.23 & 0.24 $\pm$ 0.18 & 0.97 $\pm$ 0.02 & 0.67 $\pm$ 0.28 \\
0.90 & 30 & 10 & 0.3 & 0.16 $\pm$ 0.01 & 0.06 $\pm$ 0.01 & 1.00 $\pm$ 0.00 & 0.96 $\pm$ 0.01 & 0.36 $\pm$ 0.21 & 0.22 $\pm$ 0.15 & 0.97 $\pm$ 0.02 & 0.55 $\pm$ 0.31 \\
0.90 & 30 & 10 & 0.7 & 0.32 $\pm$ 0.03 & 0.12 $\pm$ 0.01 & 1.00 $\pm$ 0.00 & 0.94 $\pm$ 0.01 & 0.44 $\pm$ 0.15 & 0.24 $\pm$ 0.13 & 0.98 $\pm$ 0.02 & 0.54 $\pm$ 0.30 \\
0.90 & 30 & 10 & 1.2 & 0.45 $\pm$ 0.04 & 0.16 $\pm$ 0.02 & 1.00 $\pm$ 0.00 & 0.93 $\pm$ 0.01 & 0.53 $\pm$ 0.10 & 0.27 $\pm$ 0.11 & 0.98 $\pm$ 0.02 & 0.54 $\pm$ 0.30 \\
0.90 & 30 & 200 & 0.3 & 0.03 $\pm$ 0.01 & 0.02 $\pm$ 0.00 & 1.00 $\pm$ 0.00 & 0.99 $\pm$ 0.00 & 0.30 $\pm$ 0.25 & 0.21 $\pm$ 0.18 & 0.98 $\pm$ 0.02 & 0.56 $\pm$ 0.32 \\
0.90 & 30 & 200 & 0.7 & 0.05 $\pm$ 0.01 & 0.03 $\pm$ 0.00 & 1.00 $\pm$ 0.00 & 0.99 $\pm$ 0.00 & 0.31 $\pm$ 0.24 & 0.22 $\pm$ 0.17 & 0.98 $\pm$ 0.02 & 0.56 $\pm$ 0.32 \\
0.90 & 30 & 200 & 1.2 & 0.08 $\pm$ 0.01 & 0.04 $\pm$ 0.01 & 1.00 $\pm$ 0.00 & 0.99 $\pm$ 0.01 & 0.32 $\pm$ 0.24 & 0.22 $\pm$ 0.17 & 0.98 $\pm$ 0.02 & 0.56 $\pm$ 0.32 \\
\bottomrule
\end{tabular}
}

\begin{table}[ht]
    \centering
        \caption{Comparison of normalized width and coverage of task-level CIs, without correlations between models.}
            \label{tab:task_no_corr}
        \resizebox{\textwidth}{!}{%
            \csname TaskNoCorr\endcsname
        }
\end{table}

\begin{table}[ht]
    \centering
        \caption{Comparison of normalized width and coverage of task-level CIs, with correlations between models.}
            \label{tab:task_corr}
        \resizebox{\textwidth}{!}{%
            \csname TaskCorr\endcsname
        }
\end{table}
\section{Real data applications - data description}\label{app:use_case_info}

\subsection{TabArena}\label{app:tabarena_info}

TabArena is a tabular data benchmarking framework open for submissions of new prediction models and datasets~\cite{erickson2025tabarena}. Performance is measured using the Elo rating system. The paper's results are available via the TabArena Python API~\footnote{https://github.com/autogluon/tabrepo}. The analysis presented in this study, are based on the results of the TabArena paper, and not on the current version of the leaderboard.

For each dataset, there are 9 or 30 folds, representing repeated 3-fold cross-validation with 3 or 10 repeats, respectively. For datasets with fewer than 2500 observations (34 datasets), the authors used 10 times repeated 3-fold outer cross-validation; for all other datasets, three repeats (17 datasets).
The benchmarks include regression, classification, and multi-class classification problems, with RMSE, ROC-AUC, and Log-Loss as the loss metrics, respectively. Additional metrics are calculated for all datasets, including train time and prediction time.

\subsection{MMLU and PromtEval}\label{app:mmlu_info}

The MMLU benchmark~\cite{hendryckstest2021} consists of multiple-choice questions across 57 subjects spanning STEM, humanities, and social sciences, with at least 100 questions per subject. It is used to measure LLMs' knowledge and problem-solving abilities. PromptEval~\cite{polo2024efficient} is a benchmark that was built upon MMLU subjects and questions, and consists of 100 prompt variations for each question with the performance of 15 LLMs on each subject, question, and prompt variant. It is publicly available (MIT license) to download from Hugging Face~\footnote {\href{https://huggingface.co/datasets/PromptEval/PromptEval_MMLU_correctness}{PromptEval dataset card}}.

For the analysis in this paper, we averaged across prompt variants to obtain an accuracy score for each model, for each question, and for each subject. As PromptEval was proposed to assess the model's sensitivity to prompt variation, it introduces two high-level sources of variability: variability across prompts and across questions. Our framework can be used to quantify both types of variability, depending on how the observed performance scores are defined.
Formally, denote the set of subjects (tasks), models, prompt variants, and questions by $Q=(T, C, V, Y)$, and by $s_{ij}^b(v_l)$ the binary score of a model $c_j$, on a subject $t_b$, a question $y_i$, and a prompt $v_l$.
The prompt-based performance score for model $c_j$:
    \begin{equation}\label{eq:avg_prompt}
    \begin{aligned}
        \basematrix^b_j(V) = \frac{1}{n_b}\sum_{i=1}^{n_b}s_{ij}^b(v_l)
    \end{aligned}
\end{equation}
where $n_b$ is the number of questions for subject $t_b$.

The question-based performance score for model $c_j$:
    \begin{equation}\label{eq:avg_example}
    \begin{aligned}
        \basematrix^b_j(Y) = \frac{1}{n_V}\sum_{l=v}^{n_V}y_{ij}^b(v_l)
    \end{aligned}
\end{equation}
where $n_V$ is the number of prompt variations for subject $t_b$, which is fixed to 100 across all subjects.

Equations~\ref{eq:avg_prompt} and~\ref{eq:avg_example} are two forms of base values that we used as input to construct the subject-level rank CIs. The width of the CIs quantifies the differences between models in sensitivity to prompt variations or in the variability within questions of the same subject.

\end{document}